\documentclass{article}

\usepackage{microtype}
\usepackage{graphicx}
\usepackage{subcaption}
\usepackage{paralist}
\usepackage{nicefrac}       

\usepackage{booktabs} 

\usepackage{hyperref}

\usepackage[normalem]{ulem}
\useunder{\uline}{\ul}{}
\usepackage{multirow}

\usepackage[accepted]{icml2023}

\usepackage{amsmath}
\usepackage{amssymb}
\usepackage{mathtools}
\usepackage{amsthm}

\usepackage[capitalize,noabbrev]{cleveref}

\theoremstyle{plain}

\usepackage[textsize=tiny]{todonotes}
\newcommand{\tabincell}[2]{\begin{tabular}{@{}#1@{}}#2\end{tabular}}

\def\I{{\bf I}}

\def\x{{\bf x}}
\def\y{{\bf y}}
\def\z{{\bf z}}

\def\m{{\bf m}}

\def\u{{\bf u}}

\def\0{{\bf 0}}
\def\1{{\bf 1}}

\def\NM{{\mathcal N}}

\def\LM{{\mathcal L}}

\def\RB{{\mathbb R}}
\def\EB{{\mathbb E}}

\def\ph{\mbox{\boldmath$\phi$\unboldmath}}

\def\tha{\mbox{\boldmath$\theta$\unboldmath}}

\def\argmax{\mathop{\rm argmax}}

\newtheorem{theorem}{Theorem}

\newtheorem{definition}{Definition}

\newtheorem{example}{Example}
\numberwithin{theorem}{section}
\numberwithin{lemma}{section}
\numberwithin{remark}{section}
\numberwithin{cor}{section}
\numberwithin{proposition}{section}
\numberwithin{definition}{section}

\newcommand{\tabref}[1]{Table~\ref{#1}}
\newcommand{\secref}[1]{Sec.~\ref{#1}}
\newcommand{\appref}[1]{Appendix~\ref{#1}}
\newcommand{\figref}[1]{Fig.~\ref{#1}}

\newcommand{\thmref}[1]{Theorem~\ref{#1}}

\newcommand{\eqnref}[1]{Eqn.~\ref{#1}}
\newcommand{\exaref}[1]{Example~\ref{#1}}
\newcommand{\defref}[1]{Definition~\ref{#1}}
\newcommand{\algref}[1]{Algorithm~\ref{#1}}
\renewcommand{\cite}{\citep}
\icmltitlerunning{Self-Interpretable Time Series Prediction with Counterfactual Explanations}

\begin{document}
\twocolumn[
\icmltitle{Self-Interpretable Time Series Prediction with Counterfactual Explanations}




\begin{icmlauthorlist}
\icmlauthor{Jingquan Yan}{ru}
\icmlauthor{Hao Wang}{ru}
\end{icmlauthorlist}

\icmlaffiliation{ru}{Department of Computer Science, Rutgers University}

\icmlcorrespondingauthor{Jingquan Yan}{jy766@rutgers.edu}

\icmlkeywords{Machine Learning, ICML}

\vskip 0.3in
]



\printAffiliationsAndNotice{}  

\begin{abstract}
Interpretable time series prediction is crucial for safety-critical areas such as healthcare and autonomous driving. Most existing methods focus on interpreting predictions by assigning important scores to segments of time series. In this paper, we take a different and more challenging route and aim at developing a self-interpretable model, dubbed Counterfactual Time Series (CounTS), which generates counterfactual and actionable explanations for time series predictions. Specifically, we formalize the problem of time series counterfactual explanations, establish associated evaluation protocols, and propose a variational Bayesian deep learning model equipped with counterfactual inference capability of time series abduction, action, and prediction. Compared with state-of-the-art baselines, our self-interpretable model can generate better counterfactual explanations while maintaining comparable prediction accuracy. Code will be available at \href{https://github.com/Wang-ML-Lab/self-interpretable-time-series}{https://github.com/Wang-ML-Lab/self-interpretable-time-series}.
\end{abstract}

\section{Introduction}
\label{introduction}

Deep learning (DL) has become increasingly prevalent, and there is naturally a growing need for understanding DL predictions in many decision-making area, such as healthcare diagnosis and public policy-making. The high-stake nature of these areas means that these DL predictions are considered trustworthy only when they can be well explained. Meanwhile, time-series data has been frequently used in these areas~\cite{zhao2021assessment,jin2022domain,yang2022artificial}, but it is always challenging to explain a time-series prediction due to the nature of temporal dependency and varying patterns over time. Moreover, time-series data often comes with confounding variables that affect both the input and output, making it even harder to explain predictions from DL models. 

On the other hand, many existing explanation methods are based on assigning importance scores for different parts of the input to explain model predictions ~\citep{LIME, SHAP, FeatureSelector, BiasAttribution, LearningPrior, Regularizing}. 
However, understanding the contribution of different input parts are usually not sufficiently informative for decision making: people often want to know what changes made to the input could have lead to a specific (desirable) prediction ~\citep{WachterCF,VisualCF,CounteRGAN}. We call such changed input that could have shifted the prediction to a specific target \emph{actionable counterfactual explanations}. Below we provide an example in the context of time series. 

\begin{example}[\textbf{Actionable Counterfactual Explanation}]\label{exa:targeted}
Suppose there is a model that takes as input a time series of breathing signal $\x\in\RB^{T}$ from a subject of age $u=60$ to predict the corresponding sleep stage as $y^{pred}=\mbox{`Awake'}\in\{\mbox{`Awake'},\mbox{`Light Sleep'},\mbox{`Deep Sleep'}\}$. 
Typical methods assign importance scores to each entry of $\x$ to explain the prediction. However, they do not provide \emph{actionable} counterfactual explanations on how to modify $\x$ to $\x^{cf}$ such that the prediction can change to $y^{cf}=\mbox{`Deep Sleep'}$. An ideal method with such capability could provide more information on why the model make specific predictions. 
\end{example}

Actionable counterfactual explanations help people understand how to achieve a counterfactual (target) output by modifying the current model input. However, such explanations may not be sufficiently informative in practice, especially under the causal effect of confounding variables which are often immutable. Specifically, some variables can hardly be changed once its value has been determined, and suggesting changing such variables are both meaningless and infeasible (e.g., a patient age and gender when modeling medical time series). This leads to a stronger requirement: a good explanation should make as few changes as possible on immutable variables; we call such explanations \emph{feasible counterfactual explanations}. Below we provide an example in the context of time series.


\begin{example}[\textbf{Feasible Counterfactual Explanation}]\label{exa:feasible}
In~\exaref{exa:targeted}, 
age $u$ is a confounder that affects both $\x$ and $y$ since elderly people (i.e., larger $u$) are more likely to have irregular breathing $\x$ and more `Awake' time (i.e., $y=\mbox{`Awake'}$) at night. 
To generate a counterfactual explanation to change $y^{pred}$ to $\mbox{`Deep Sleep'}$, typical methods tend to suggest decreasing the age $u$ from $60$ to $50$, which is \emph{infeasible} (since age cannot be changed in practice). An ideal method would first infer the age $u$ and search for a \emph{feasible} counterfactual explanation $\x^{cf}$ that could change $y^{pred}$ to $\mbox{`Deep Sleep'}$ while keeping $u$ unchanged. 
\end{example}

In this paper, we propose a self-interpretable time series prediction model, dubbed Counterfactual Time Series (CounTS), which can both (1) perform time series predictions and (2) provide actionable and feasible counterfactual explanations for its predictions. Under common causal structure assumptions, our method is guaranteed to identify the causal effect between the input and output in the presence of exogenous (confounding) variables, thereby improving the generated counterfactual explanations' feasibility. Our contribution is summarized as follows:

\begin{itemize}
    \item We identify the actionability and feasibility requirements for generating counterfactual explanations for time series models and develop the first general self-interpretable method, dubbed CounTS, that satisfies such requirements. 
    \item We provide theoretical guarantees that CounTS can identify the causal effect between the time series input and output in the presence of exogenous (confounding) variables, thereby improving feasibility in the generated explanations. 
    \item Experiments on both synthetic and real-world datasets show that compared to state-of-the-art methods, CounTS significantly improves performance for generating 
    counterfactual explanations while still maintaining comparable prediction accuracy. 
    
\end{itemize}

\section{Related Work}
\label{relatedwork}

\textbf{Interpretation Methods for Neural Networks.} 
Various attribution-based interpretation methods have been proposed in recent years. Some methods focused on local interpretation \citep{LIME, SHAP, MAPLE, FeatureSelector, BiasAttribution} while others are designed for global interpretation \citep{ACE, MAME}. The main idea is to assign attribution, or importance scores, to the input features in terms of their impact on the prediction (output). For example, such importance scores can be computed using gradients of the prediction with respect to the input \citep{GradCAM, GradSHAP, DEEPLIFT, IntegratedGradient}. 
Some interpretation methods are specialized for time series data; these include perturbation-based \citep{SeriesSaliency}, rule-based \citep{LIMREF}, and attention-based methods \citep{UncertaintyAware,TFT}. One typical method, Feature Importance in Time (FIT), evaluates the importance of the input data based on the temporal distribution shift and unexplained distribution shift \cite{FIT}. However, these methods can only produce importance scores of the input features for the current prediction and therefore cannot generate counterfactual explanations (see~\secref{sec:exp} and~\appref{sec:app_add} for empirical results). 


\textbf{Counterfactual Explanations for Time Series Models.} 
There also works that generate counterfactual explanations for time series models. 
\cite{CAP} proposed an association-rule algorithm to explain time series prediction by finding the frequent pairs of timestamps and generating counterfactual examples. 
\cite{CounteRGAN} proposed a general explanation framework that generates counterfactual examples using residual generative adversarial networks (RGAN); it can be adapted for time series models. However, these works either fail to generate realistic counterfactual explanations (due to discretization error) or fail to generate feasible counterfactual explanations for time series models. In contrast, our CounTS as a principled variational causal method~\cite{ICL,GenInt,gupta2021correcting} can naturally generate realistic and feasible counterfactual explanations. Such advantages are empirically verified in~\secref{sec:exp}. 

\textbf{Bayesian Deep Learning and Variational Autoencoders.} 
Our work is also related to the broad categories of variational autoencoders (VAEs)~\cite{VAE} (which use inference networks to approximate posterior distributions) and Bayesian deep learning (BDL)~\cite{CDL,BDL,BDLThesis,RPPU,BIN,BDLSurvey,ZESRec} models (which use a deep component to process high-dimensional signals and a task-specific/graphical component to handle conditional/causal dependencies). \cite{OrphicX} proposed the first VAE-based model for generating causal explanations for graph neural networks.  \cite{CEVAE,DeepSCM} proposed the first VAE-based models for performing causal inference and estimating treatment effect. However, none of them addressed the problem of counterfactual explanation, which involves solving an inverse problem to obtain the optimal counterfactual input. In contrast, our CounTS is the first VAE-based model to address this challenge, with theoretical guarantees and promising empirical results. From the perspective of BDL~\cite{BDL,BDLSurvey}, CounTS uses deep neural networks to process high-dimension signals (i.e., the deep component in~\cite{BDL}) and uses a Bayesian network to handle the conditional/causal dependencies among variables (i.e., the task-specific or graphical component in~\cite{BDL}). Therefore, CounTS is also the first BDL model for generating counterfactual explanations. 



\section{Preliminaries}
\label{preliminaries}

\textbf{Causal Model.}
Following the definition in \cite{Causality}, a causal model is described by a 3-tupple $M = \left \langle  U, V, F\right \rangle$. $U$ is a set of exogenous variables $\{u_1,\dots,u_m\}$ that is not determined by any other variables in this causal model. $V$ is a set of endogenous variables $\{v_1,\dots,v_n\}$ that are determined by variables in $U \cup V$. We assume the causal model can be factorized according to a directed graph where each node represents one variable. $F$ is a set of functions $\{f_1,\dots,f_n\}$ describing the generative process of $V$: 
\begin{equation*}
v_i = f_i(pa_i, u_i), i=1,\dots,n,
\end{equation*}
where $pa_i$ denotes direct parent nodes of $v_i$. 

\textbf{Counterfactual Inference.} 
Counterfactual inference is interested in questions like ``observing that $X=x$ and $Y=y$, what would be probability that $Y=y^{cf}$ if the input $X$ had been $x^{cf}$?". Formally, given a causal model $\left \langle  U, V, F\right \rangle$ where $Y, X \in V$, counterfactual inference proceeds in three steps~\cite{Causality}:
\begin{compactenum}
    \item \textbf{Abduction.} Calculate the posterior distribution of $u$ given the observation $X=x$ and $Y=y$, i.e., $P(u|X=x,Y=y)$. 
    \item \textbf{Action.} Perform causal intervention on the variable $X$, i.e., $do(X=x^{cf})$. 
    \item \textbf{Prediction.} Calculate the counterfactual probability $P(Y_{X=x^{cf}}(u)=y^{cf})$ with respect to the posterior distribution of $P(u|X=x, Y=y)$.
\end{compactenum}
Putting the three steps together, we have
\begin{align*}
&P(Y_{X=x^{cf}}=y^{cf}|X=x,Y=y) \\
= &\sum\nolimits_u P(Y_{X=x^{cf}}(u)=y^{cf})P(u|x,y),
\end{align*}
where $x^{cf}$ can be an counterfactual explanation describing what would have shifted the outcome $Y$ from $y$ to $y^{cf}$. 

\section{Method}
In this section, we formalize the problem of counterfactual explanations for time series prediction and describe our proposed method for the problem. 

\textbf{Problem Setting.} 
We focus on generating counterfactual explanations for predictions from time series models. We assume the model takes as input a multivariate time series $\x_i \in \RB^{D \times T}$ and predicts the corresponding label $\y_i$, which can be a categorical label, a real value, or a time series $\y_i\in\RB^{T}$. Given a specific input $\x_i$ and the model's prediction $\y_i^{pred}$, our goal is to explain the model by finding a counterfactual time series $\x_i^{cf}\neq \x_i$ that could have lead the model to an alternative (counterfactual) prediction $\y_i^{cf}$. 


\begin{figure}[t]
\centering
\includegraphics[width=0.48\textwidth]{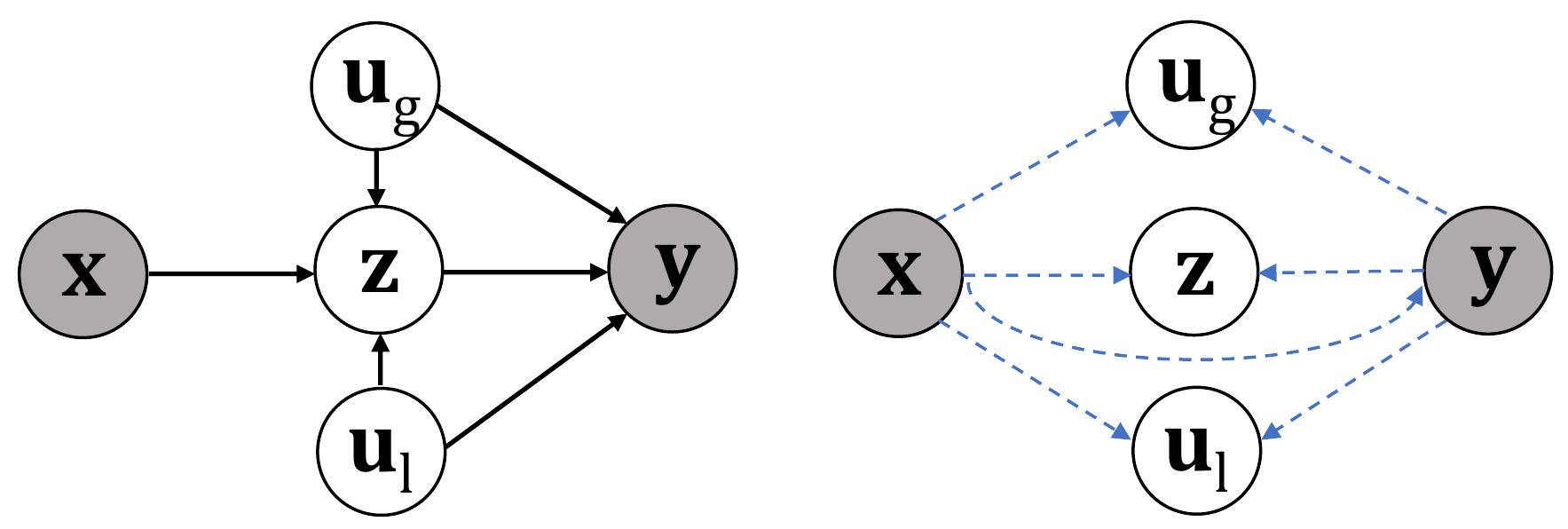}
\caption{\label{fig:causal_graph}\textbf{Left:} The causal graph and generative model for CounTS. \textbf{Right:} Inference model for CounTS.}
\end{figure}

\subsection{Learning of CounTS}

\textbf{Causal Graph and Generative Model for CounTS.} 
Our self-interpretable model CounTS is based on the causal graph in~\figref{fig:causal_graph}(left), where $\x\in\RB^{D\times T_{in}}$ is the input time series, $\y$ is the label, $\z\in\RB^{H_z\times T_{mid}}$ is the representation of $\x$, $\u_l\in\RB^{H_l\times T_{mid}}$ is the local exogenous variable, and $\u_g\in\RB^{H_g}$ is the global exogenous variable. Both $\u_l$ and $\u_g$ are confounder variables; $\u_l$ can take different values in different time steps, while $\u_g$ is shared across all time steps.  

This causal graph assumes the following factorization of the generative model $p_\theta(\y,\u_l,\u_g,\z|\x)$:
\begingroup\makeatletter\def\f@size{7.5}\check@mathfonts
\def\maketag@@@#1{\hbox{\m@th\normalsize\normalfont#1}}%
\begin{align}
p_\theta(\y,\u_l,\u_g,\z|\x) = p_\theta(\u_l,\u_g)  p_\theta(\z | \u_l,\u_g, \x)p_\theta(\y| \u_l,\u_g, \z),\label{eq:p_factor}
\end{align}
\endgroup
where $\tha$ denotes the collection of parameters for the generative model, and $p_\theta(\u_l,\u_g)=p_\theta(\u_l)p_\theta(\u_g)$. Here $p_\theta(\z | \u_l,\u_g, \x)$ and $p_\theta(\y| \u_l,\u_g, \z)$ are the encoder and the predictor, respectively. Specifically, we have
\begin{align*}
p_\theta(\u_l) &= \NM(\0,\I),~~~~p_\theta(\u_g) = \NM(\0,\I), \\
p_\theta(\z | \u_l,\u_g, \x) &= \NM(\mu_z(\u_l,\u_g, \x; \tha), \sigma_z(\u_l,\u_g,\x; \tha)), \\
p_\theta(\y | \u_l,\u_g, \z) &= \NM(\mu_y(\u_l,\u_g, \z; \tha), \sigma_y(\u_l,\u_g,\z; \tha)),
\end{align*}
where $\mu_z(\cdot; \cdot)$, $\sigma_z(\cdot; \cdot)$, $\mu_y(\cdot; \cdot)$, $\sigma_y(\cdot; \cdot)$ are neural networks parameterized by $\tha$. Note that for classification models the predictor $p_\theta(\y | \u_l,\u_g, \z)$ becomes a categorical distribution $Cat(f_y(\u_l,\u_g, \z; \tha))$, where $f_y(\cdot;\cdot)$ is a neural network. 

\textbf{Inference Model for CounTS.} 
We use an inference model $q_\phi(\y, \u_l,\u_g,\z|\x)$ to approximate the posterior distribution of the latent variables, i.e., $p_\theta(\u_l,\u_g,\z|\x)$. As shown in~\figref{fig:causal_graph}(right), we factorize $q_\phi(\y, \u_l,\u_g,\z|\x)$ as
\begingroup\makeatletter\def\f@size{9.5}\check@mathfonts
\def\maketag@@@#1{\hbox{\m@th\normalsize\normalfont#1}}%
\begin{align}
q_\phi(\y, \u_l,\u_g,\z|\x) &= q_\phi(\y|\x)q_\phi(\u_l,\u_g| \x, \y)q_\phi(\z | \x, \y),\label{eq:variational}\\
q_\phi(\u_l,\u_g| \x, \y)&=q_{\phi}(\u_l|\x, \y)q_{\phi}(\u_g|\x, \y)
\end{align}
\endgroup
where $\ph$ is the collection of the inference model's parameters, and . We parameterized each factor in~\eqnref{eq:variational} as
\begin{align*}
q_{\phi}(\y|\x) &=\NM(\mu_y(\x; \ph), \sigma_y^2(\x; \ph)), \\
q_{\phi}(\u_l|\x, \y) &= \NM(\mu_{u_l}(\x, \y; \ph), \sigma_{u_l}^2(\x, \y; \ph)), \\
q_{\phi}(\u_g|\x, \y) &= \NM(\mu_{u_g}(\x, \y; \ph), \sigma_{u_g}^2(\x, \y; \ph)), \\
q_{\phi}(\z|\x, \y) &= \NM(\mu_z(\x, \y; \ph), \sigma_z^2(\x, \y; \ph)),
\end{align*}
where $\mu_{\cdot}(\cdot;\cdot)$ and $\sigma_{\cdot}(\cdot;\cdot)$ denote neural networks with $\ph$ as their parameters.

\textbf{Evidence Lower Bound.} 
Our CounTS uses the evidence lower bound (ELBO) $\LM_{ELBO}$ of the log likelihood $\log p(\y|\x)$ as an objective to learn the generative and inference models. 
Maximizing the ELBO is equivalent to learning the optimal variational distribution $q_\phi(\u_l,\u_g,\z|\x)=\int q_\phi(\y, \u_l,\u_g,\z|\x) d\y$ that best approximates the posterior distribution of the label and latent variables $p_\theta(\u_l,\u_g,\z|\x)$. Specifically, we have
\begin{align}
     \LM_{ELBO} = &~\EB_{q_\phi(\y, \u_l,\u_g,\z|\x)}[p_{\theta}(\y, \u_l,\u_g,\z|\x)] \nonumber\\
     &- \EB_{q_\phi(\y, \u_l,\u_g,\z|\x)}[q_\phi(\y, \u_l,\u_g,\z|\x)].\label{eq:elbo_simple}
\end{align}
Note that different from typical ELBOs, we explicitly involves $\y$ and use $q_\phi(\y, \u_l,\u_g,\z|\x)$ rather than $q_\phi(\u_l,\u_g,\z|\x)$  (see~\appref{sec:app_elbo} for details); this is to expose the factor $q_\phi(\y|\x)$ in~\eqnref{eq:variational} to allow for additional supervision on $\y$ (more details below). With the factorization in~\eqnref{eq:p_factor} and~\eqnref{eq:variational}, we can decompose \eqnref{eq:elbo_simple} as:
\begingroup\makeatletter\def\f@size{9.2}\check@mathfonts
\def\maketag@@@#1{\hbox{\m@th\normalsize\normalfont#1}}%
\begin{align}
\LM&_{ELBO}=~\EB_{q_\phi(\y|\x)}\EB_{q_\phi(\u_l,\u_g, \z| \x,\y)}[p_\theta(\y| \u_l,\u_g, \z)]\label{eq:elbo_y_}\\
&-\EB_{q_\phi(\y,\u_l,\u_g|\x)}\big[KL[q_{\phi}(\z|\x,\y)||p_\theta(\z | \u_l,\u_g, \x)]\big]\label{eq:elbo_z}\\
&- \EB_{q_\phi(\y|\x)}\big[KL[q_\phi(\u_l,\u_g| \x, \y)||p_\theta(\u_l,\u_g)]\big]\label{eq:elbo_u}\\
&- \EB_{q_{\phi}(\y|\x)}[q_{\phi}(\y|\x)],\label{eq:elbo_y}
\end{align}
\endgroup
where each term is computed with our neural network parameterization; \figref{fig:arch}(left) shows the network structure. Below we briefly discuss the intuition of each term. 
\begin{compactenum}
\item[(1)] \textbf{\eqnref{eq:elbo_y_} predicts} the label $\y$ using $\u_l$, $\u_g$, and $\z$ inferred from $\x$ ($\y$ is marginalized out).
\item[(2)] \textbf{\eqnref{eq:elbo_z} regularizes} the inference model $q_{\phi}(\z|\x,\y)$ to get closer to the generative model $p_\theta(\z | \u_l,\u_g, \x)$. 
\item[(3)] \textbf{\eqnref{eq:elbo_u} regularizes} $q_{\phi}(\u_l,\u_g|\x,\y)$ using the prior distribution $p(\u_l,\u_g)$. 
\item[(4)] \textbf{\eqnref{eq:elbo_y} regularizes} $q_{\phi}(\y|\x)$ by maximizing its entropy, preventing it from collapsing to deterministic solutions. 
\end{compactenum}

\begin{figure*}[t]
    \centering
    \includegraphics[width=0.99\textwidth]{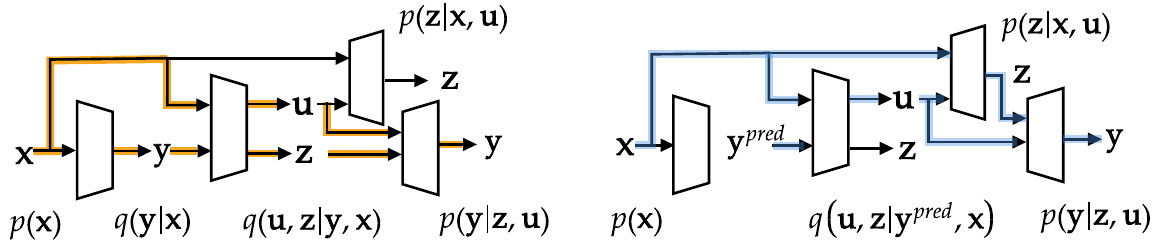}
    \vspace{-10pt}
    \caption{\label{fig:arch}Network structure. We omit subscripts of $p_\theta$ and $q_\phi$ for clarity. \textbf{Left:} CounTS makes \emph{predictions} on $\y$ using the yellow branch. \textbf{Right:} Given the current $\x$ and $\y^{pred}$, CounTS generates \emph{counterfactual explanation} $\x^{cf}$ for any target label $\y^{cf}$ using the blue branch.}
    \vspace{-11pt}
  \end{figure*}

\textbf{Final Objective Function.} 
Inspired by~\cite{CEVAE}, we include an additional term ($N$ is the training set size)
\begin{align}
\LM_y = \sum\nolimits_{i=1}^N\log q(\y_i|\x_i)
\end{align}
to supervise $q(\y|\x)$ using label $\y$ during training; this helps produce more accurate estimation for $\u_l$, $\u_g$, $\z$ using $q_\phi(\u_l,\u_g|\x,\y)$ and $q_\phi(\z|\x,\y)$. 
The final objective function then becomes
\begin{align}
\LM_{CounTS} = \LM_{ELBO} + \lambda \LM_y,\label{eq:final}
\end{align}
where $\lambda$ is a hyperparameter balancing both terms.

\subsection{Inference Using CounTS}


After learning both the generative model (\eqnref{eq:p_factor}) and the inference model (\eqnref{eq:variational}) by maximizing $\LM_{CounTS}$ in~\eqnref{eq:final}, our \textbf{self-interpretable} CounTS can perform inference to (1) make \textbf{predictions} on $\y$ using the yellow branch in~\figref{fig:arch}(left), and (2) generate \textbf{counterfactual explanation} $\x^{cf}$ for any target label $\y^{cf}$ using the blue branch in~\figref{fig:arch}(right).

\subsubsection{Prediction}
We use the yellow branch in~\figref{fig:arch}(left) to \emph{predict} $\y$. Specifically, given an input time series $\x$, the encoder will first infer the \emph{initial label} $\y$ using $q_\phi(\y|\x)$ and then infer $(\u_l,\u_g)$ and $\z$ from $\x$ and $\y$ using $q_\phi(\u_l,\u_g,\z|\x,\y)$. $(\u_l,\u_g)$ and $\z$ are then fed into the predictor $p_\theta(\y|\u_l,\u_g,\z)$ to infer the \emph{final} label $\y$. Formally, CounTS predict the label as
\begin{align}
\y^{pred} = \EB_{q_\phi(\y|\x)}\EB_{q_\phi(\u_l,\u_g, \z| \x,\y)}[p_\theta(\y| \u_l,\u_g, \z)].
\end{align}
Empirically, we find that directly using the means of $q_\phi(\y|\x)$ and $q_\phi(\u_l,\u_g, \z| \x,\y)$ as input to $p_\theta(\y| \u_l,\u_g, \z)$ already achieves satisfactory accuracy. 




\subsubsection{Generating Counterfactual Explanation}
We use the blue branch in~\figref{fig:arch}(right) to generate counterfactual explanations via counterfactual inference. Our goal is to find the optimal counterfactual explanation $\x^{cf}$ defined below.

\begin{definition}[\textbf{Optimal Counterfactual Explanation}]\label{def:counter}
    Given a factual observation $\x$ and prediction $\y^{pred}$, the optimal counterfactual explanation $\x^{cf}$ for the counterfactual outcome for $\y^{cf}$ is
\begin{align*}
\x^{cf} = \argmax\nolimits_{\x'}~ p(Y_{\x=\x'}(\u)=\y^{cf}),
\end{align*}
where $\u=(\u_l,\u_g)$ and the counterfactual likelihood is defined as
\begingroup\makeatletter\def\f@size{9.5}\check@mathfonts
\def\maketag@@@#1{\hbox{\m@th\normalsize\normalfont#1}}%
\begin{align}
&p(Y_{\x=\x^{\prime}}(\u)=\y^{cf}) \label{eq:cf} \\
&= \sum_{\u} p\left(\y=\y^{cf} | do\left(\x=\x^{\prime}\right), \u\right) p(\u | \x=\x, \y=\y^{pred}). \nonumber
\end{align}
\endgroup
\end{definition}

In words, we search for the optimal $\x^{cf}$ that would have shifted the model prediction from $\y^{pred}$ to $\y^{cf}$ while keeping $(\u_l,\u_g)$ unchanged. Since the definition of counterfactual explanations in~\defref{def:counter} involves causal inference with the intervention on $\x$, we need to first \emph{identify} the causal probability $p(\y=\y^{cf} | do(\x=\x^{\prime}), \u)$ using observational probability, i.e., removing `do' in the equation. The theorem below shows that this is achievable. 



\begin{theorem}[\textbf{Identifiability}]\label{thm:do}
    Given the posterior distribution of exogenous variable $p(\u_l,\u_g|\x,\y)$, the effect of action $p(\y=\y^{cf} | do(\x=\x^{\prime}), \u_l,\u_g)$ can be identified using $\EB_{p(\z | \x^{\prime}, \u_l,\u_g)} [p\left(\y^{cf} | \z, \u_l,\u_g\right)]$.
\end{theorem}

See \appref{sec:proof} for the proof. With 
\thmref{thm:do}, we can rewrite \eqnref{eq:cf} as
\begin{align}
\LM_{cf} = \EB_{p(\u | \x=\x, \y=\y^{pred})}\EB_{p(\z | \x^{\prime}, \u)}  [p(\y^{cf} | \z, \u) ],\label{eq:final}
\end{align}
where $\u=(\u_l,\u_g)$ and $p(\u | \x=\x, \y=\y^{pred})$ is approximated by $q_\phi(\u | \x=\x, \y=\y^{pred})$. We use Monte Carlo estimates to compute the expectation in \eqnref{eq:cf} and \eqnref{eq:final}, iteratively compute the gradient $\frac{\partial\LM_{cf}}{\partial\x'}$ (via back-propagation) to search for the optimal $\x'$ in a way similar to~\cite{BIN,ReverseAttack}, and use it as $\x^{cf}$ (see the complete algorithm in \appref{sec:app_algo}). 


\section{Experiments}\label{sec:exp}
In this section, we evaluate our CounTS and existing methods on two \emph{synthetic} and three \emph{real-world datasets}. For each dataset, we evaluate different methods in terms of three metrics: (1) \textbf{prediction accuracy}, (2) \textbf{counterfactual accuracy}, and (3) \textbf{counterfactual change ratio}, with the last one as the most important metric. These metrics take different forms for different datasets (see details in~\secref{sec:toy}-\ref{sec:real}). 

\subsection{Baselines and Implementations}
We compare our CounTS with state-of-the-art methods for generating explanations for deep learning models, including Regularized Gradient Descent (\textbf{RGD})~\cite{RGD}, Gradient-weighted Class Activation Mapping (\textbf{GradCAM})~\cite{GradCAM}, 
Gradient SHapley Additive exPlanations (\textbf{GradSHAP})~\cite{GradSHAP}, Local Interpretable Model-agnostic Explanations (\textbf{LIME})~\cite{LIME}, Feature Importance in Time (\textbf{FIT})~\cite{FIT}, Case-crossover APriori (\textbf{CAP})~\cite{CAP}, and Counterfactual Residual Generative Adversarial Network (\textbf{CounteRGAN})~\cite{CounteRGAN} (see~\appref{sec:app_baseline} for more details). Note that among these baselines, only RGD and CounteRGAN can generate actionable explanations. Other baselines, including FIT (which is designed for time series models), only provide importance scores as explanations; therefore some evaluation metrics may not be applicable for them (shown as `-' in tables). 

All methods above are implemented with PyTorch \cite{PyTorch}. For fair comparison, the prediction model in all the baseline explanation methods has the same neural network architecture as the inference module in our CounTS. See the Appendix for more details on the architecture, training, and inference. 



\begin{table}[t]
\vskip -0.25cm
\caption{\textbf{Results on the toy dataset.} We mark the best result with \textbf{bold face} and the second best results with \underline{underline}.}
\label{tab:toy_result}
\setlength{\tabcolsep}{2pt}
\resizebox{0.48\textwidth}{!}{%
\begin{tabular}{cccccccccc}
\toprule
 & CounteRGAN & RGD & GradCAM & GradSHAP  & LIME & FIT & CAP & CounTS (Ours) \\ \midrule
Pred. Accuracy (\%) $\uparrow$ & \multicolumn{7}{c}{-------------------------- \quad\textbf{85.19}\quad --------------------------} & {\ul 83.26} \\
CCR $\uparrow$ & {\ul 1.25} & 1.21 & 1.09 & 1.13  & 1.2 & 1.15 & 0.97 & \textbf{1.33} \\
Counterf. Accuracy (\%) $\uparrow$ & \textbf{78.75} & {\ul 77.96} & - & -  & - & 45.62 & - & 70.28 \\ \bottomrule
\end{tabular}%
}
\end{table}


\subsection{Toy Dataset}\label{sec:toy}
\textbf{Dataset Description.} 
We designed a toy dataset where the label is affected by only part of the input. A good counterfactual explanation should only modify this part of the input while keeping the other part unchanged. 
Following the causal graph in \figref{fig:causal_graph}(left),  we have input $\x\in\RB^{12}$ with each entry independently sampled from $\mathcal{N} (\mu_x , \sigma^2_x)$ and an exogenous variable $u \in \RB$ (as the confounder) sampled from $\mathcal{N} ( \mu_u , \sigma^2_u)$. To indicate which part of $\x$ affects the label, we introduce a mask vector $\m \in \RB^{12}$ with first $6$ entries $\m_{1:6}$ set to $1$ and the last 6 entries $\m_{7:12}$ set to $0$. We then generate $\z\in \RB^{12}$ as $\z = \u \cdot(\m \odot \x) $ and the label $y \sim Bern(\sigma(\z^\top\1 + u))$ where $\sigma$ is the sigmoid function. Here $\x$'s first $6$ entries $\x_{1:6}$ is label-related and the last $6$ entries $\x_{7:12}$ is label-agnostic. 

\textbf{Evaluation Metrics.} 
We use three evaluation metrics:
\begin{compactitem}
\item \textbf{Prediction Accuracy.} 
This is the percentage of time series correctly predicted ($y^{pred} = y$) in the test set. 
\item \textbf{Counterfactual Accuracy.} 
For a prediction model $f$, the generated counterfactual explanation $\x^{cf}$, and the target label $y^{cf}$, counterfactual accuracy is the percentage of time series where $\x^{cf}$ successfully change the model's prediction to $y^{cf}$ (i.e., $f(\x^{cf}) = y^{cf}$). 
\item \textbf{Counterfactual Change Ratio (CCR).} 
This measures how well the counterfactual explanation $\x^{cf}$ changes the label-related input $\x_{1:6}$ while keeping the label-agnostic input $\x_{7:12}$ unchanged. Formally we use the average ratio of $\frac{\|\x_{1:6}^{cf} - \x_{1:6}\|_1}{\|\x_{7:12}^{cf} - \x_{7:12}\|_1}$ across the test set. 
\end{compactitem}

Note that CCR is the most important metric among the three since our main focus is to generate actionable and feasible counterfactual explanations. 

\textbf{Quantitative Results.} 
\tabref{tab:toy_result} compares our CounTS with the baselines in terms of the three metrics. 
CounTS outperforms all the baselines in terms of CCR with minimal prediction accuracy loss. This shows that our CounTS successfully identifies and fixes the exogenous variables $u$ and $\m$ to generate actionable and feasible counterfactual explanations while still achieving prediction accuracy comparable to baselines. 
Note that actionable methods (i.e., CounTS, RGD, and CounteRGAN) outperforms importance score methods (i.e., FIT, LIME, and GradSHAP). This is expected because importance-score-based explanations can only explain the original prediction and therefore fail to infer what change on $\x$ will shift the prediction from $y^{pred}$ to $y^{cf}$. 

Our counterfactual accuracy is lower than CounteRGAN and RGD methods. This is reasonable since these baselines do not infer and fix the posterior distribution $p(\u|\x,\y)$ and are therefore more flexible for the generator (or back-propagation) to modify their input to push $y^{pred}$ closer to the target $y^{cf}$. 
However, such flexibility comes at the cost of low feasibility, reflected in their poor CCR performance. 

\begin{table}[t]
\vskip -0.25cm
\caption{\textbf{Results on the \emph{Spike} dataset.} We mark the best result with \textbf{bold face} and the second best results with \underline{underline}.}
\setlength{\tabcolsep}{2pt}
\label{tab:spike}
\resizebox{0.49\textwidth}{!}{%
\begin{tabular}{cccccccccc}
\toprule
 &  & CounteRGAN & RGD & GradCAM & GradSHAP  & LIME & FIT & CAP & CounTS (Ours) \\ \midrule
Pred. MSE $\downarrow$ &  & \multicolumn{7}{c}{-------------------------- \quad{\underline{0.128}}\quad-------------------------- } & \textbf{0.117} \\
\multirow{2}{*}{CCR $\uparrow$} & 2 Active & 2.206 & {\ul2.217} & 2.043 & 1.989  & 1.872 & 1.863 & 1.614 & \textbf{2.322} \\
 & 1 Active & {\ul 0.705} & 0.683 & 0.654 & 0.615  & 0.473 & 0.55 & 0.497 & \textbf{0.730} \\
Counterf. MSE $\downarrow$ &  & \textbf{0.074} & \textbf{0.074} & - & -  & - & 0.394 & - & {\ul 0.103} \\ \bottomrule
\end{tabular}%
}
\end{table}

\begin{figure*}[th]
    \centering
        \includegraphics[width=1\textwidth]{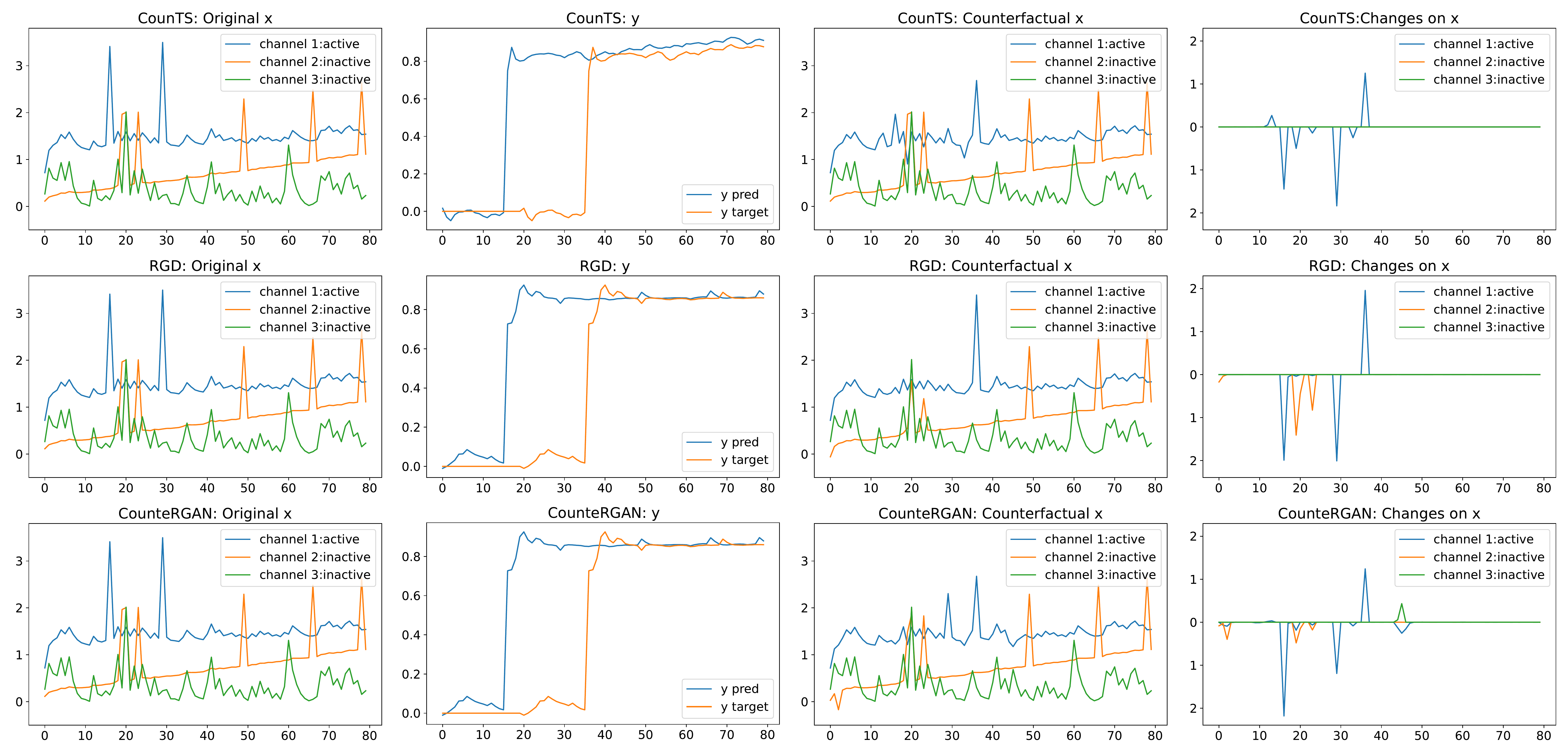}
        \vspace{-25pt}
    \caption{\label{fig:syntheticcase}Qualitative results on the \emph{Spike} dataset. \textbf{Column 1} shows the original input $\x$ for CounTS, RGD, and CounteRGAN, respectively. \textbf{Column 2} shows the predictions $\y^{pred}$ and counterfactual (target) labels $\y^{cf}$. The prediction for RGD and CounteRGAN are identical since they use the same prediction model. \textbf{Column 3} shows the counterfactual input $\x^{cf}$. \textbf{Column 4} shows the changes on input, i.e., $\x^{cf} - \x$. FIT \emph{cannot} provide actionable explanations (see the importance score generated by FIT in \appref{sec:app_add}). }
    \vspace{-15pt}
\end{figure*}

\subsection{Spike Dataset}\label{sec:synthetic}
\textbf{Dataset Description.} 
Inspired by~\cite{FIT}, we construct the \emph{Spike} synthetic dataset. In the dataset, each time series contains $3$ channel, i.e., $\x\in\RB^{3\times T}$. Each channel is a independent non–linear auto-regressive moving average (NARMA) sequence with randomly distributed spikes. The label sequence $\y\in\RB^{T}$ starts at $0$; it \emph{may} switch to $1$ as soon as there is a spike observed in any of the $3$ channels and remain $1$ until the last timestamp. A $3$-dimensional exogenous variable $\u\in\{0,1\}^{3}$ 
determines whether the spike in the channel can affect the final output; if so, we say the channel is \emph{active}. Each of the three entries in $\u$, i.e., is sampled independently from three different Bernoulli distributions with parameters $0.8$, $0.4$, and $0$, respectively 
(see more details on the dataset in~\appref{sec:app_dataset}). 



\textbf{Evaluation Metrics.} 
We use three evaluation metrics:
\begin{compactitem}
\item \textbf{Prediction MSE.} 
We use the mean square error (MSE) $\frac{1}{N}\sum_i^{N}\|\y_i^{pred} - \y_i\|_2^2$ to measure the prediction error in the test set with $N$ time series. 
\item \textbf{Counterfactual MSE.} 
Similar to~\secref{sec:toy}, for a prediction model $f$, the generated counterfactual explanation $\x^{cf}$, and the target label $\y^{cf}$, counterfactual MSE is defined as $\frac{1}{N}\sum_i^{N} \|f(\x_i^{cf})) - \y_i^{cf}\|_2^2$; it measures how successfully $\x^{cf}$ changes the model's prediction to $y^{cf}$. 
\item \textbf{CCR.} 
For a given input $\x_i$, we set the target counterfactual label $\y_i^{cf}$ by shifting $\y^{pred}$ by $20$ timestamps to the right. 
If at timestamp $t$, there is a spike in an active channel in $\x_i$ triggering the output $\y_i^{pred}$ to switch from $0$ to $1$, an ideal counterfactual explanation $\x_i^{cf}$ should (1) suppress all spikes between $[t, t+20)$, (2) create a new spike at $t+20$ timestamp in all active channels of the original input $\x_i$, and (3) keep all inactive channels unchanged. Therefore the counterfactual change ratio can be defined as (with $N$ time series):
$
CCR = \frac{1}{N} \sum_{i=1}^N  \frac{\|\m_i\odot (\x_i - \x_i^{cf})\|_1}{\|(\1-\m_i)\odot (\x_i - \x_i^{cf})\|_1}, 
$
where $\m_i=[\u,\u,\dots,\u]\in\RB^{3\times T}$ repeats $\u$ in each time step to mask inactive channels. 
\end{compactitem}



Note that the scale of CCR will depend on the number of active channels, we therefore report our results for time series with $1$ and $2$ active channels. Higher CCR indicates better performance. 


\textbf{Quantitative Results.} 
\tabref{tab:spike} shows the results for different methods in the \emph{Spike}. Similar to the toy dataset, our CounTS outperforms all baselines in terms of CCR, and actionable methods (i.e., CounTS, RGD, and CounteRGAN) outperforms importance score methods (i.e., FIT, LIME, and GradSHAP) thanks to the former's capability of modifying the input to shift the prediction towards the counterfactual target label. 

Interestingly, besides promising performance in terms of CCR, our CounTS can also improve prediction performance, achieving lower prediction MSE. This is potentially due to CounTS's ability to model the exogenous variable $\u$ that decides whether a spike in a specific channel affects the label $\y$. Similar to the toy datset, we notice both CounteRGAN and RGD methods have lower counterfactual MSE (both at $0.074$) than CounTS because they do not need to infer and fix the exogenous variable $\u$ (i.e., the mask) and therefore have more flexibility to modify the input $\x$. 
Since both CounteRGAN and RGD method are unaware of the mask, they suffer from lower CCR and produce more unwanted modification in inactive channels (more details below). 





\textbf{Qualitative Results.} 
\figref{fig:syntheticcase} shows an example time series as a case study. In this example, only channel $1$ (blue) is active and thus the spike in channel $1$ at timestamp $t=16$ will flip the output $\y$ from $0$ to $1$. Both CounTS's and RGD's predictions $\y^{pred}$ (blue in Column 1) are very close to the ground truth. 

Following the evaluation protocol above, we set our counterfactual (target) label $\y^{cf}$ by shifting $\y^{pred}$ by $20$ timestamps to the right and padding $0$s on the left (yellow). 
We should expect an ideal counterfactual explanation to have no spike in $t=[16, 36)$ and a new spike at $t=36$ in channel 1. Since channel 2 and 3 are inactive, there should not be any changes. We can see all actionable methods (i.e., CounTS, CounteRGAN, and RGD) try to reduce the spikes in $t=[16,36)$. However, CounteRGAN fails to remove the spike at $t=29$; both RGD and CounteRGAN makes undesirable changes in inactive channels 2 and 3. Compared to baselines, our CounTS shows a significant advantage at inferring the active channel and suppressing the changes in inactive channels. This demonstrates that CounTS is more actionable and feasible than the existing time-series explanation methods. 
Note that FIT is \emph{not} an actionable explanation method and therefore can only provide importance scores to explain the current prediction $\y^{pred}$ (see \appref{sec:app_add} \figref{fig:fitcase} for details).

\begin{table}[t]
\vspace{-3pt}
\setlength{\tabcolsep}{3.0pt}
\caption{{\textbf{Average CCR for both intra-dataset and cross-dataset settings.} The average CCR is calculated over 6 counterfactual action settings (A$\rightarrow$D, L$\rightarrow$D, R$\rightarrow$D, R$\rightarrow$A, D$\rightarrow$A, and L$\rightarrow$A) for every model and dataset setting. The \emph{last row} are the average CCR over all 6 dataset settings for every method. }}
\label{tab:realccr}
\resizebox{0.48\textwidth}{!}{%
\begin{tabular}{@{}ccccccccc@{}}
\toprule
 & CounteRGAN & RGD & GradCAM & GradSHAP & LIME & FIT & CAP & CounTS (Ours) \\ \midrule
$\rightarrow$SOF & {\ul 1.153} & 1.135 & 1.126 & 0.908 & 0.932 & 1.072 & 1.034 & \textbf{1.313} \\
$\rightarrow$SHHS & 1.231 & {\ul 1.241} & 1.056 & 0.975 & 1.020 & 1.193 & 0.929 & \textbf{1.390} \\
$\rightarrow$MESA & {\ul 1.285} & 1.188 & 1.031 & 1.025 & 1.007 & 1.142 & 1.031 & \textbf{1.377} \\ \midrule \midrule
SOF & {\ul 1.126} & 1.088 & 0.887 & 0.996 & 1.078 & 1.053 & 0.963 & \textbf{1.281} \\
SHHS & {\ul 1.133} & 1.111 & 0.890 & 1.068 & 1.059 & 1.067 & 1.000 & \textbf{1.333} \\
MESA & 1.163 & {\ul 1.166} & 0.883 & 1.042 & 0.987 & 1.089 & 0.962 & \textbf{1.329} \\ \midrule
Average & {\ul 1.182} & 1.155 & 0.979 & 1.002 & 1.014 & 1.103 & 0.987 & \textbf{1.337} \\ \bottomrule
\end{tabular}%
}
\end{table}

\subsection{Real-World Datasets}\label{sec:real}
\textbf{Dataset Description.} 
We evaluate our model on three real-world medical datasets: Sleep Heart Health Study (\emph{SHHS})~\cite{SHHS}, Multi-ethnic Study of Atherosclerosis (\emph{MESA})~\cite{MESA}, and Study of Osteoporotic Fractures (\emph{SOF})~\cite{SOF}, each containing subjects' full-night breathing breathing signals. The breathing signals are divided into $30$-second segments; each segment has one of the sleep stages (`Awake', `Light Sleep', `Deep Sleep', `Rapid Eye Movement (REM)') as its label. In total, there are $2{,}651$, $2{,}055$ and $453$ patients in SHHS, MESA, and SOF, respectively, with an average of $1{,}043$ breathing signal segments (approximately $8.7$ hours). 



\begin{table}[t]
\vspace{-3pt}
\caption{\textbf{CCR for each counterfactual action setting in \emph{SHHS}.}}
\vskip 0.1pt
\label{tab:shhs}
\resizebox{0.48\textwidth}{!}{%
\begin{tabular}{cccccccc}
\toprule
 & A$\rightarrow$D & L$\rightarrow$D & R$\rightarrow$D & R$\rightarrow$A & D$\rightarrow$A & L$\rightarrow$A & Average \\ \midrule
CounteRGAN & 1.119 & 1.144 & 1.055 & 1.298 & 1.125 & 1.059 & {\ul 1.133} \\
RGD & 1.006 & 1.130 & 0.985 & 1.180 & {\ul 1.155} & 1.201 & 1.111 \\
GradCAM & 0.882 & 0.790 & 0.859 & 0.954 & 0.882 & 0.977 & 0.890 \\
GradSHAP & 1.100 & 1.105 & 0.903 & 1.171 & 1.100 & 1.027 & 1.068 \\
LIME & 0.855 & \textbf{1.178} & 1.067 & 1.125 & 0.855 & {\ul 1.273} & 1.059 \\
FIT & 0.948 & 1.012 & 1.215 & {\ul 1.405} & 0.948 & 0.877 & 1.067 \\
CAP & 0.885 & 0.829 & {\ul 1.236} & 0.967 & 0.885 & 1.199 & 1.000 \\ 
CounTS (Ours) & \textbf{1.244} & {\ul 1.159} & \textbf{1.349} & \textbf{1.515} & \textbf{1.324} & \textbf{1.407} & \textbf{1.333}\\ \bottomrule
\end{tabular}%
}
\end{table}

Corresponding to our causal graph in \figref{fig:causal_graph}(left), $\x$ is the original breathing signal, $\z$ is signal patterns (representation) learned from the encoder, $\u$ can be `gender', `age', and even `dataset index' (`dataset index' will be used during cross-dataset prediction; more details later), and $\y$ is the sleep stage label. 
Intuitively, `gender', `age', and `dataset index' (due to different experiment instruments and signal measurement) can have causal effect on both the breathing signal patterns (subjects of different ages can have different breathing frequencies and magnitudes) and sleep stages (elderly subjects have less `Deep Sleep' at night). 
    

\textbf{Evaluation Metrics.} 
For brevity, we use abbreviations `A', `R', `L', and `D' to represent the four sleep stages ‘Awake’, ‘Rapid Eye Movement (REM)’, ‘Light
Sleep’, ‘Deep Sleep’, respectively. 
We use three evaluation metrics:
\begin{compactitem}
\item \textbf{Prediction Metric} and \textbf{Counterfactual Metric} are similar to those in the toy classification dataset. 
\item \textbf{CCR.} Since ground-truth confounders are not available in real-world datasets, we propose to compute CCR on two consecutive 30-second segments with different sleep stages. For example, 
given the two segments $[\x_{i,1},\x_{i,2}]$ with the corresponding two predicted sleep stages $[\y_{i,1}^{pred},\y_{i,2}^{pred}] = [A, D]$, we set the counterfactual labels $[\y_{i,1}^{cf},\y_{i,2}^{cf}] = [A, A]$; we refer to this setting as D$\rightarrow$A. An ideal counterfactual explanation $[\x_{i,1}^{cf},\x_{i,2}^{cf}]$ should shift the model prediction from $[A,D]$ to $[A,A]$ while keeping the first segment unchanged (i.e., smaller $\|\x_{i,1}-\x_{i,1}^{cf}\|_1$). CCR can therefore be computed as (with $N$ segment pairs in the test set): 
$
CCR = \frac{1}{N} \sum_{i=1}^N  \frac{\|\x_{i,2} - \x_{i,2}^{cf}\|_1}{\|\x_{i,1} - \x_{i,1}^{cf}\|_1}. 
$
\end{compactitem}
\vspace{-5pt}
In the experiments, we focus on cases where the target label $\y_{i,2}^{cf}$ is either `Awake' or `Deep Sleep' (two extremes of sleep stages), and evaluated $6$ \emph{counterfactual action settings}, namely A$\rightarrow$D, L$\rightarrow$D, R$\rightarrow$D, R$\rightarrow$A, D$\rightarrow$A, and L$\rightarrow$A. 



\textbf{Cross-Dataset and Intra-Dataset Settings.} 
Since the three datasets are collected with different devices and procedures, we consider the dataset index as an additional exogenous variable in $\u$, which has causal effect on both the learned representation $\z$ and the sleep stage $\y$. This leads to two different settings: the intra-dataset setting (without the dataset index confounder) and the cross-dataset setting (with the dataset index confounder).
\begin{compactitem}
    \item \textbf{Intra-Dataset Setting.}  Training and test sets are from same dataset, and we do not involve the dataset index as part of the exogenous variable (confounder) $\u$. 
    \item \textbf{Cross-Dataset Setting.} We choose two of the datasets as the source datasets, with the remaining one as the target dataset (e.g., \emph{MESA}+\emph{SHHS}$\rightarrow$\emph{SOF}). We use all source datasets (e.g., \emph{SHHS} and \emph{MESA}) and $10\%$ of the target dataset (e.g., \emph{SOF}) as the training set and use the remaining $90\%$ of the target dataset as the test set. We treat the dataset index as part of $\u$, which is predicted during both training and inference. 
\end{compactitem}

\textbf{Quantitative Results.} 
{Table~\ref{tab:realccr} shows the average CCR of different methods for both intra-dataset and cross-dataset settings. Results show that our CounTS outperforms the baselines in all the dataset settings in terms of the average CCR. This demonstrates CounTS's capability of generating feasible and actionable explanations in complex real-world datasets. {\tabref{tab:shhs} shows the detailed CCR for each counterfactual action setting in \emph{SHHS}, where CounTS is leading in almost all counterfactual action settings. Detailed results for other datasets can be found in \tabref{tab:sof}$\sim$\ref{tab:tosof} of \appref{sec:app_add}.}

\begin{table}[t]
\vspace{-5pt}
\caption{\textbf{Prediction accuracy for both intra-dataset and cross-dataset settings.} All baselines (e.g., RGD, CounteRGAN, and FIT) explain the same prediction model and therefore share the same prediction accuracy. `$\rightarrow$\emph{SHHS}' means `\emph{MESA}+\emph{SOF}$\rightarrow$\emph{SHHS}' and similarly for `$\rightarrow$\emph{MESA}' and `$\rightarrow$\emph{SOF}'. }
\label{tab:realaccu}
\vskip 0.1cm
\setlength{\tabcolsep}{2pt}
\resizebox{0.48\textwidth}{!}{%
\begin{tabular}{ccccccc}
\toprule
Method & $\rightarrow$\emph{SOF} & $\rightarrow$\emph{SHHS} & $\rightarrow$\emph{MESA} & \emph{SOF} & \emph{SHHS} & \emph{MESA} \\ \midrule
 All Baselines & 0.702 & 0.747 & \textbf{0.768} & 0.729 & \textbf{0.801} & \textbf{0.813} \\
CounTS (Ours)  & \textbf{0.719} & \textbf{0.753} & 0.741 & \textbf{0.732} & 0.789 & 0.799\\ \bottomrule
\end{tabular}%
}
\end{table}

\begin{table}[t]
\vspace{-5pt}

\caption{\textbf{Average counterfactual accuracy for both intra-dataset and cross-dataset settings.} The average is calculated over all 6 counterfactual action settings.}
\label{tab:real_cf_accu}
\vskip 0.1cm
\setlength{\tabcolsep}{1pt}
\resizebox{0.48\textwidth}{!}{%
\begin{tabular}{@{}cccccccc@{}}
\toprule
 & $\rightarrow$SOF & $\rightarrow$SHHS & $\rightarrow$MESA & SOF & SHHS & MESA & Average \\ \midrule
RGD & {\ul 0.913} & \textbf{0.887} & \textbf{0.907} & \textbf{0.862} & \textbf{0.907} & \textbf{0.892} & \textbf{0.895} \\
CounteRGAN & 0.897 & {\ul 0.868} & 0.880 & 0.844 & 0.888 & 0.868 & 0.874 \\
FIT & 0.331 & 0.327 & 0.346 & 0.248 & 0.322 & 0.294 & 0.311 \\
CounTS (Ours) & \textbf{0.916} & \textbf{0.887} & {\ul 0.889} & {\ul 0.852} & {\ul 0.903} & {\ul 0.887} & {\ul 0.889} \\ \bottomrule
\end{tabular}%
}
\end{table}

\tabref{tab:realaccu} shows the prediction accuracy for both intra-dataset and cross-dataset settings. Similar to~\tabref{tab:toy_result} and~\tabref{tab:spike}, all baselines (e.g., RGD, CounteRGAN, and FIT) explain the same prediction model and therefore share the same prediction accuracy. Results show that CounTS achieves prediction accuracy comparable to the original prediction model. ~\tabref{tab:real_cf_accu} shows the average counterfactual accuracy for different dataset settings over $6$ counterfactual action settings. As in the toy and \emph{Spike} datasets, RGD achieves higher counterfactual accuracy because it does not need to infer and fix the exogenous variable $\u$ and hence enjoys more flexibility when modifying the input $\x$; this often leads to undesirable changes (e.g., breathing frequency and amplitude which are related to the subject's age), less feasible counterfactual explanations, and therefore worse CCR performance. In contrast, breathing frequency and amplitude are potentially captured by $\u_l$ and $\u_g$ in CounTS and kept unchanged (see qualitative results below). {Detailed counterfactual accuracy for every counterfactual action setting and every dataset setting can be found in \appref{sec:app_add}.}  

\begin{figure}[t]
\vskip -0.3cm
\centering
\hspace*{-7pt}
        \includegraphics[width=0.49\textwidth]{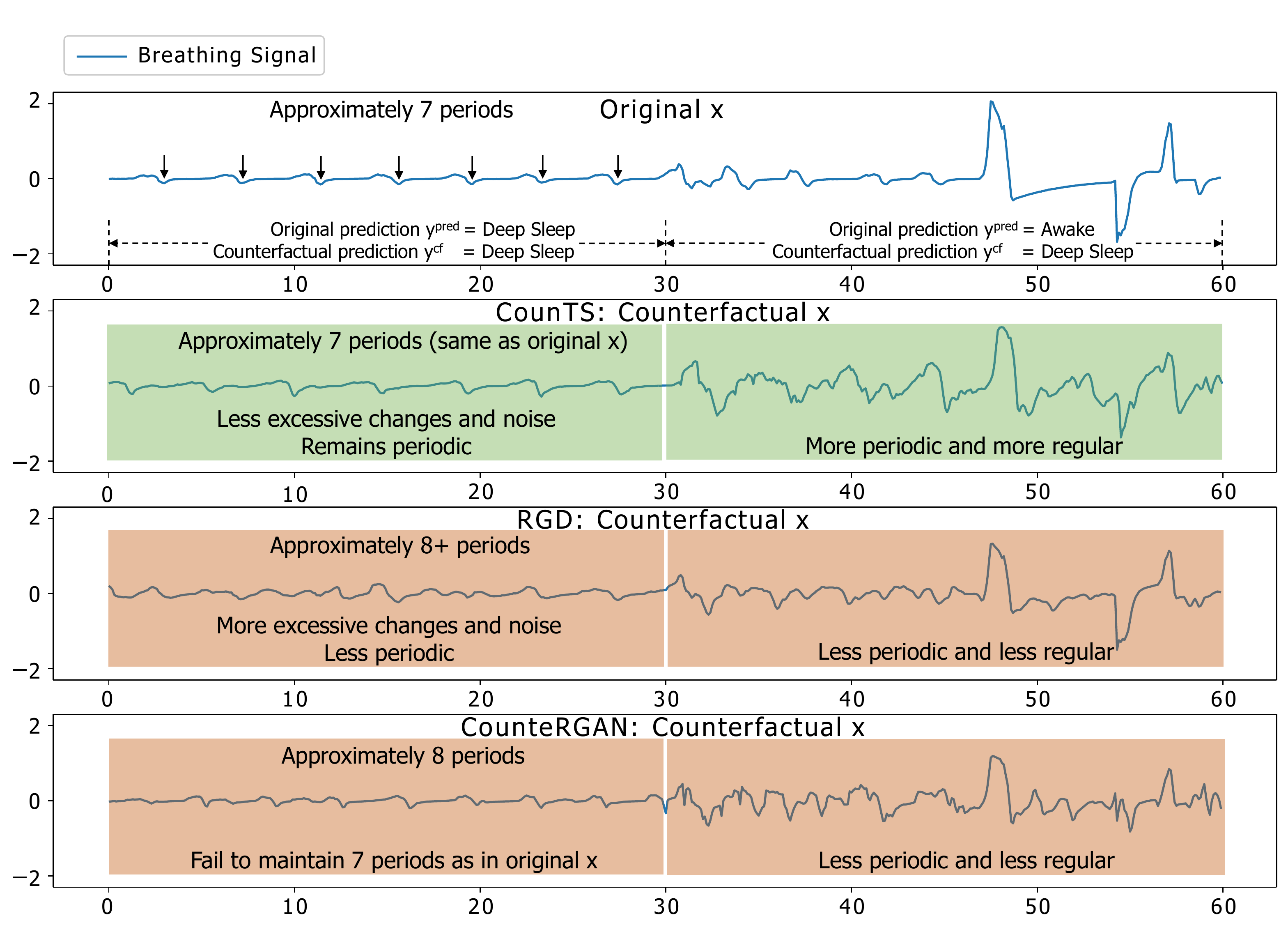}
        \vspace{-25pt}
    \caption{\label{fig:realcase}Example on real-world dataset from \emph{MESA} intra-dataset experiment. For this example, $\y^{pred}$ is [D,A] and target $\y^{cf}$ is [D,D] since we assume `Deep Sleep' and `Active' have the most different patterns. FIT \emph{cannot} provide actionable explanations (see the importance score generated by FIT in \appref{sec:app_add}).}
    \vskip 0.2cm
\end{figure}

\textbf{Qualitative Analysis: Problem Setting.}
As some background on sleep staging, note that a patient's breathing signal will be much more periodic in `Deep Sleep' than in `Awake'. \figref{fig:realcase} shows an example time series from the MESA dataset as a case study. It contains two consecutive segments of breathing signals ($60$ seconds in total). The first segment ($0\sim 30$ seconds) is regular and periodic, and therefore both CounTS and baselines correctly predict its label as `Deep Sleep' (`D'); the second segment ($30\sim 60$ seconds) is irregular and therefore both CounTS and baselines correctly predict its label as `Awake' (`A'). The goal is to generate a counterfactual explanation such that the model predicts the \textit{second} segment as `Deep Sleep'.

\textbf{Qualitative Analysis: Ideal Counterfactual Explanation.}
Since a `Deep Sleep' segment should be more periodic, an \textit{ideal} counterfactual explanation should 
keep the first segment ($0\sim 30$ seconds) \textit{unchanged} and maintain its \textit{periodicity}, 
make the second segment ($30\sim 60$ seconds) \textit{more periodic} (the pattern for `Deep Sleep'), and 
keep the breathing frequency (number of breathing cycles) unchanged throughout both segments (this is related to "feasibility" since a patient usually has the same breathing frequency during `Deep Sleep'). 

\textbf{Qualitative Analysis: Detailed Results.} 
\figref{fig:realcase} compares the generated counterfactual explanations from CounTS, RGD, and CounteRGAN, and we can see that our CounTS's explanation is closer to the ideal case. Specifically:

\begin{compactenum}
    \item \textbf{The First 30-Second Segment:} In the original breathing signal, the first 30-second segment ($0\sim 30$ seconds) is periodic and contains 7 breathing cycles (i.e., 7 periods). Our CounTS managed to keep the first 30-second segment signal mostly unchanged, maintaining its periodicity and 7 breathing cycles. In contrast, RGD changes the first segment to the point that its periodicity is mostly lost. CounteRGAN can maintain the first segment's periodicity; however it increases the number breathing cycles from 7 to 8, making it \textit{infeasible}.
    \item \textbf{The Second 30-Second Segment:} In the original breathing signal, the second 30-second segment ($30\sim 60$ seconds) is irregular since the patient is in the `Awake' sleep stage. CounTS managed to make the `Awake' (i.e., `Active') signal much more periodic, successfully steering the prediction to `Deep Sleep'. In contrast, RGD and CounteRGAN failed to make the second 30-second segment ($30\sim 60$ seconds) more periodic. 
    \item \textbf{Capturing and Maintaining Global Temporal Patterns:} Meanwhile, we can observe that CounTS's explanation in the second 30-second segment ($30\sim 60$ seconds) has a much more similar breathing frequency (measured by the number of breathing cycles per minute) with the first 30-second segment ($0\sim 30$ seconds), compared to RGD and CounteRGAN. This observation shows that CounTS is capable of capturing global temporal pattern (the breathing frequency of the individual) and keeping such exogenous variables unchanged in its explanation thanks to the global exogenous variable $\mathbf{u}_g$ in our model.
\end{compactenum}

\section{Conclusion}
In this paper, we identified the actionability and feasibility requirements for time series models counterfactual explanations and proposed the first self-interpretable time series prediction model, CounTS, that satisfies both requirements. Our theoretical analysis shows that CounTS is guaranteed to identify the causal effect between time series input and output in the presence of confounding variables, 
thereby generating counterfactual explanations in the causal sense. 
Empirical results showed that our method achieves competitive prediction accuracy on time series data and is capable of generating more actionable and feasible counterfactual explanations. Interesting future work includes further reducing the computation complexity during counterfactual inference, handling uncertainty in the explanation~\cite{TFUncertainty}, and extending our method to multimodal data.

\section*{Acknowledgement}
We would like to thank the reviewers/AC for the constructive comments to improve the paper. JY and
HW are partially supported by NSF Grant IIS-2127918. We also thank National Sleep Research Resource (NSRR) data repository and the researchers for sharing SHHS~\citep{SHHS1,SHHS2}, MESA~\citep{SHHS1, MESA1} and SOF~\citep{SHHS1, SOF1}. 

\nocite{langley00}

\bibliography{paper}

\begin{thebibliography}{51}
\providecommand{\natexlab}[1]{#1}
\providecommand{\url}[1]{\texttt{#1}}
\expandafter\ifx\csname urlstyle\endcsname\relax
  \providecommand{\doi}[1]{doi: #1}\else
  \providecommand{\doi}{doi: \begingroup \urlstyle{rm}\Url}\fi

\bibitem[Chen et~al.(2018)Chen, Song, Wainwright, and Jordan]{FeatureSelector}
Chen, J., Song, L., Wainwright, M., and Jordan, M.
\newblock Learning to explain: An information-theoretic perspective on model
  interpretation.
\newblock In \emph{International Conference on Machine Learning}, pp.\
  883--892. PMLR, 2018.

\bibitem[Chen et~al.(2014)Chen, Wang, Zee, Lutsey, Javaheri, Alcántara,
  Jackson, Williams, and Redline]{MESA1}
Chen, X., Wang, R., Zee, P., Lutsey, P., Javaheri, S., Alcántara, C., Jackson,
  C., Williams, M., and Redline, S.
\newblock Racial/ethnic differences in sleep disturbances: The multi-ethnic
  study of atherosclerosis (mesa).
\newblock \emph{Sleep}, 38, 11 2014.
\newblock \doi{10.5665/sleep.4732}.

\bibitem[Cummings et~al.(1990)Cummings, Black, Nevitt, Browner, Cauley, Genant,
  Mascioli, Scott, Seeley, Steiger, et~al.]{SOF}
Cummings, S.~R., Black, D.~M., Nevitt, M.~C., Browner, W.~S., Cauley, J.~A.,
  Genant, H.~K., Mascioli, S.~R., Scott, J.~C., Seeley, D.~G., Steiger, P.,
  et~al.
\newblock Appendicular bone density and age predict hip fracture in women.
\newblock \emph{JAMA}, 263\penalty0 (5):\penalty0 665--668, 1990.

\bibitem[Dhaou et~al.(2021)Dhaou, Bertoncello, Gourv{\'e}nec, Garnier, and
  Le~Pennec]{CAP}
Dhaou, A., Bertoncello, A., Gourv{\'e}nec, S., Garnier, J., and Le~Pennec, E.
\newblock Causal and interpretable rules for time series analysis.
\newblock In \emph{Proceedings of the 27th ACM SIGKDD Conference on Knowledge
  Discovery \& Data Mining}, pp.\  2764--2772, 2021.

\bibitem[Ding et~al.(2022)Ding, Ma, Deoras, Wang, and Wang]{ZESRec}
Ding, H., Ma, Y., Deoras, A., Wang, Y., and Wang, H.
\newblock Zero-shot recommender systems.
\newblock In \emph{ICLR Workshop on Deep Generative Models for Highly
  Structured Data}, 2022.

\bibitem[Ghorbani et~al.(2019)Ghorbani, Wexler, Zou, and Kim]{ACE}
Ghorbani, A., Wexler, J., Zou, J.~Y., and Kim, B.
\newblock Towards automatic concept-based explanations.
\newblock \emph{Advances in Neural Information Processing Systems}, 32, 2019.

\bibitem[Goyal et~al.(2019)Goyal, Wu, Ernst, Batra, Parikh, and Lee]{VisualCF}
Goyal, Y., Wu, Z., Ernst, J., Batra, D., Parikh, D., and Lee, S.
\newblock Counterfactual visual explanations.
\newblock In \emph{International Conference on Machine Learning}, pp.\
  2376--2384. PMLR, 2019.

\bibitem[Gupta et~al.(2021)Gupta, Wang, Lipton, and Wang]{gupta2021correcting}
Gupta, S., Wang, H., Lipton, Z., and Wang, Y.
\newblock Correcting exposure bias for link recommendation.
\newblock In \emph{International Conference on Machine Learning}, pp.\
  3953--3963. PMLR, 2021.

\bibitem[Heo et~al.(2018)Heo, Lee, Kim, Lee, Kim, Yang, and
  Hwang]{UncertaintyAware}
Heo, J., Lee, H.~B., Kim, S., Lee, J., Kim, K.~J., Yang, E., and Hwang, S.~J.
\newblock Uncertainty-aware attention for reliable interpretation and
  prediction.
\newblock In Bengio, S., Wallach, H., Larochelle, H., Grauman, K.,
  Cesa-Bianchi, N., and Garnett, R. (eds.), \emph{Advances in Neural
  Information Processing Systems}, volume~31. Curran Associates, Inc., 2018.

\bibitem[Huang et~al.(2019)Huang, Wang, and Mak]{RPPU}
Huang, H., Wang, H., and Mak, B.
\newblock Recurrent poisson process unit for speech recognition.
\newblock In \emph{AAAI}, volume~33, pp.\  6538--6545, 2019.

\bibitem[Jin et~al.(2022)Jin, Park, Maddix, Wang, and Wang]{jin2022domain}
Jin, X., Park, Y., Maddix, D., Wang, H., and Wang, Y.
\newblock Domain adaptation for time series forecasting via attention sharing.
\newblock In \emph{International Conference on Machine Learning}, pp.\
  10280--10297. PMLR, 2022.

\bibitem[Kingma \& Welling(2013)Kingma and Welling]{VAE}
Kingma, D.~P. and Welling, M.
\newblock Auto-encoding variational bayes.
\newblock \emph{arXiv preprint arXiv:1312.6114}, 2013.

\bibitem[Lim et~al.(2021)Lim, Ar{\i}k, Loeff, and Pfister]{TFT}
Lim, B., Ar{\i}k, S.~{\"O}., Loeff, N., and Pfister, T.
\newblock Temporal fusion transformers for interpretable multi-horizon time
  series forecasting.
\newblock \emph{International Journal of Forecasting}, 37\penalty0
  (4):\penalty0 1748--1764, 2021.

\bibitem[Lin et~al.(2022)Lin, Lan, Wang, and Li]{OrphicX}
Lin, W., Lan, H., Wang, H., and Li, B.
\newblock Orphicx: A causality-inspired latent variable model for interpreting
  graph neural networks.
\newblock In \emph{CVPR}, 2022.

\bibitem[Louizos et~al.(2017)Louizos, Shalit, Mooij, Sontag, Zemel, and
  Welling]{CEVAE}
Louizos, C., Shalit, U., Mooij, J.~M., Sontag, D., Zemel, R., and Welling, M.
\newblock Causal effect inference with deep latent-variable models.
\newblock \emph{Advances in neural information processing systems}, 30, 2017.

\bibitem[Lundberg \& Lee(2017{\natexlab{a}})Lundberg and Lee]{GradSHAP}
Lundberg, S.~M. and Lee, S.-I.
\newblock A unified approach to interpreting model predictions.
\newblock In Guyon, I., Luxburg, U.~V., Bengio, S., Wallach, H., Fergus, R.,
  Vishwanathan, S., and Garnett, R. (eds.), \emph{Advances in Neural
  Information Processing Systems}, volume~30. Curran Associates, Inc.,
  2017{\natexlab{a}}.

\bibitem[Lundberg \& Lee(2017{\natexlab{b}})Lundberg and Lee]{SHAP}
Lundberg, S.~M. and Lee, S.-I.
\newblock A unified approach to interpreting model predictions.
\newblock \emph{Advances in neural information processing systems}, 30,
  2017{\natexlab{b}}.

\bibitem[Mao et~al.(2021{\natexlab{a}})Mao, Chiquier, Wang, Yang, and
  Vondrick]{ReverseAttack}
Mao, C., Chiquier, M., Wang, H., Yang, J., and Vondrick, C.
\newblock Adversarial attacks are reversible with natural supervision.
\newblock In \emph{ICCV}, 2021{\natexlab{a}}.

\bibitem[Mao et~al.(2021{\natexlab{b}})Mao, Gupta, Cha, Wang, Yang, and
  Vondrick]{GenInt}
Mao, C., Gupta, A., Cha, A., Wang, H., Yang, J., and Vondrick, C.
\newblock Generative interventions for causal learning.
\newblock In \emph{CVPR}, 2021{\natexlab{b}}.

\bibitem[Mi et~al.(2022)Mi, Wang, Tian, and Shavit]{TFUncertainty}
Mi, L., Wang, H., Tian, Y., and Shavit, N.
\newblock Training-free uncertainty estimation for neural networks.
\newblock In \emph{AAAI}, 2022.

\bibitem[Natesan~Ramamurthy et~al.(2020)Natesan~Ramamurthy, Vinzamuri, Zhang,
  and Dhurandhar]{MAME}
Natesan~Ramamurthy, K., Vinzamuri, B., Zhang, Y., and Dhurandhar, A.
\newblock Model agnostic multilevel explanations.
\newblock \emph{Advances in neural information processing systems},
  33:\penalty0 5968--5979, 2020.

\bibitem[Nemirovsky et~al.(2022)Nemirovsky, Thiebaut, Xu, and
  Gupta]{CounteRGAN}
Nemirovsky, D., Thiebaut, N., Xu, Y., and Gupta, A.
\newblock Countergan: Generating counterfactuals for real-time recourse and
  interpretability using residual gans.
\newblock In Cussens, J. and Zhang, K. (eds.), \emph{Proceedings of the
  Thirty-Eighth Conference on Uncertainty in Artificial Intelligence}, volume
  180 of \emph{Proceedings of Machine Learning Research}, pp.\  1488--1497.
  PMLR, 01--05 Aug 2022.

\bibitem[Pan et~al.(2021)Pan, Hu, and Chen]{SeriesSaliency}
Pan, Q., Hu, W., and Chen, N.
\newblock Two birds with one stone: Series saliency for accurate and
  interpretable multivariate time series forecasting.
\newblock In \emph{IJCAI}, pp.\  2884--2891, 2021.

\bibitem[Paszke et~al.(2019)Paszke, Gross, Massa, Lerer, Bradbury, Chanan,
  Killeen, Lin, Gimelshein, Antiga, Desmaison, Kopf, Yang, DeVito, Raison,
  Tejani, Chilamkurthy, Steiner, Fang, Bai, and Chintala]{PyTorch}
Paszke, A., Gross, S., Massa, F., Lerer, A., Bradbury, J., Chanan, G., Killeen,
  T., Lin, Z., Gimelshein, N., Antiga, L., Desmaison, A., Kopf, A., Yang, E.,
  DeVito, Z., Raison, M., Tejani, A., Chilamkurthy, S., Steiner, B., Fang, L.,
  Bai, J., and Chintala, S.
\newblock Pytorch: An imperative style, high-performance deep learning library.
\newblock In \emph{Advances in Neural Information Processing Systems 32}, pp.\
  8024--8035. Curran Associates, Inc., 2019.

\bibitem[Pawlowski et~al.(2020)Pawlowski, Coelho~de Castro, and
  Glocker]{DeepSCM}
Pawlowski, N., Coelho~de Castro, D., and Glocker, B.
\newblock Deep structural causal models for tractable counterfactual inference.
\newblock \emph{Advances in Neural Information Processing Systems},
  33:\penalty0 857--869, 2020.

\bibitem[Pearl(2009)]{Causality}
Pearl, J.
\newblock \emph{Causality}.
\newblock Cambridge University Press, 2 edition, 2009.

\bibitem[Plumb et~al.(2018)Plumb, Molitor, and Talwalkar]{MAPLE}
Plumb, G., Molitor, D., and Talwalkar, A.~S.
\newblock Model agnostic supervised local explanations.
\newblock \emph{Advances in neural information processing systems}, 31, 2018.

\bibitem[Plumb et~al.(2020)Plumb, Al-Shedivat, Cabrera, Perer, Xing, and
  Talwalkar]{Regularizing}
Plumb, G., Al-Shedivat, M., Cabrera, {\'A}.~A., Perer, A., Xing, E., and
  Talwalkar, A.
\newblock Regularizing black-box models for improved interpretability.
\newblock \emph{Advances in Neural Information Processing Systems},
  33:\penalty0 10526--10536, 2020.

\bibitem[Quan et~al.(1998)Quan, Howard, Iber, Kiley, Nieto, O'Connor, Rapoport,
  Redline, Robbins, Samet, and Wahl]{SHHS2}
Quan, S., Howard, B., Iber, C., Kiley, J., Nieto, F., O'Connor, G., Rapoport,
  D., Redline, S., Robbins, J., Samet, J., and Wahl, P.
\newblock The sleep heart health study: Design, rationale, and methods.
\newblock \emph{Sleep}, 20:\penalty0 1077--85, 01 1998.
\newblock \doi{10.1093/sleep/20.12.1077}.

\bibitem[Quan et~al.(1997)Quan, Howard, Iber, Kiley, Nieto, O'connor, Rapoport,
  Redline, Robbins, Samet, et~al.]{SHHS}
Quan, S.~F., Howard, B.~V., Iber, C., Kiley, J.~P., Nieto, F.~J., O'connor,
  G.~T., Rapoport, D.~M., Redline, S., Robbins, J., Samet, J.~M., et~al.
\newblock The sleep heart health study: design, rationale, and methods.
\newblock \emph{Sleep}, 20\penalty0 (12):\penalty0 1077--1085, 1997.

\bibitem[Rajapaksha \& Bergmeir(2022)Rajapaksha and Bergmeir]{LIMREF}
Rajapaksha, D. and Bergmeir, C.
\newblock Limref: Local interpretable model agnostic rule-based explanations
  for forecasting, with an application to electricity smart meter data.
\newblock \emph{arXiv preprint arXiv:2202.07766}, 2022.

\bibitem[Ribeiro et~al.(2016)Ribeiro, Singh, and Guestrin]{LIME}
Ribeiro, M.~T., Singh, S., and Guestrin, C.
\newblock "why should i trust you?" explaining the predictions of any
  classifier.
\newblock In \emph{Proceedings of the 22nd ACM SIGKDD international conference
  on knowledge discovery and data mining}, pp.\  1135--1144, 2016.

\bibitem[Selvaraju et~al.(2017)Selvaraju, Cogswell, Das, Vedantam, Parikh, and
  Batra]{GradCAM}
Selvaraju, R.~R., Cogswell, M., Das, A., Vedantam, R., Parikh, D., and Batra,
  D.
\newblock Grad-cam: Visual explanations from deep networks via gradient-based
  localization.
\newblock In \emph{Proceedings of the IEEE international conference on computer
  vision}, pp.\  618--626, 2017.

\bibitem[Shrikumar et~al.(2017)Shrikumar, Greenside, and Kundaje]{DEEPLIFT}
Shrikumar, A., Greenside, P., and Kundaje, A.
\newblock Learning important features through propagating activation
  differences.
\newblock In \emph{International conference on machine learning}, pp.\
  3145--3153. PMLR, 2017.

\bibitem[Spira et~al.(2008)Spira, Blackwell, Stone, Redline, Cauley,
  Ancoli-Israel, and Yaffe]{SOF1}
Spira, A.~P., Blackwell, T., Stone, K.~L., Redline, S., Cauley, J.~A.,
  Ancoli-Israel, S., and Yaffe, K.
\newblock Sleep-disordered breathing and cognition in older women.
\newblock \emph{Journal of the American Geriatrics Society}, 56\penalty0
  (1):\penalty0 45--50, 2008.

\bibitem[Sundararajan et~al.(2017)Sundararajan, Taly, and
  Yan]{IntegratedGradient}
Sundararajan, M., Taly, A., and Yan, Q.
\newblock Axiomatic attribution for deep networks.
\newblock In \emph{International conference on machine learning}, pp.\
  3319--3328. PMLR, 2017.

\bibitem[Tonekaboni et~al.(2020)Tonekaboni, Joshi, Campbell, Duvenaud, and
  Goldenberg]{FIT}
Tonekaboni, S., Joshi, S., Campbell, K., Duvenaud, D.~K., and Goldenberg, A.
\newblock What went wrong and when? instance-wise feature importance for
  time-series black-box models.
\newblock In Larochelle, H., Ranzato, M., Hadsell, R., Balcan, M., and Lin, H.
  (eds.), \emph{Advances in Neural Information Processing Systems}, volume~33,
  pp.\  799--809. Curran Associates, Inc., 2020.

\bibitem[Wachter et~al.(2017{\natexlab{a}})Wachter, Mittelstadt, and
  Russell]{RGD}
Wachter, S., Mittelstadt, B., and Russell, C.
\newblock Counterfactual explanations without opening the black box: Automated
  decisions and the gdpr.
\newblock \emph{Harv. JL \& Tech.}, 31:\penalty0 841, 2017{\natexlab{a}}.

\bibitem[Wachter et~al.(2017{\natexlab{b}})Wachter, Mittelstadt, and
  Russell]{WachterCF}
Wachter, S., Mittelstadt, B.~D., and Russell, C.
\newblock Counterfactual explanations without opening the black box: Automated
  decisions and the gdpr.
\newblock \emph{Cybersecurity}, 2017{\natexlab{b}}.

\bibitem[Wang(2017)]{BDLThesis}
Wang, H.
\newblock \emph{Bayesian Deep Learning for Integrated Intelligence: Bridging
  the Gap between Perception and Inference}.
\newblock PhD thesis, Hong Kong University of Science and Technology, 2017.

\bibitem[Wang \& Yeung(2016)Wang and Yeung]{BDL}
Wang, H. and Yeung, D.-Y.
\newblock Towards bayesian deep learning: A framework and some existing
  methods.
\newblock \emph{TDKE}, 28\penalty0 (12):\penalty0 3395--3408, 2016.

\bibitem[Wang \& Yeung(2020)Wang and Yeung]{BDLSurvey}
Wang, H. and Yeung, D.-Y.
\newblock A survey on bayesian deep learning.
\newblock \emph{CSUR}, 53\penalty0 (5):\penalty0 1--37, 2020.

\bibitem[Wang et~al.(2015)Wang, Wang, and Yeung]{CDL}
Wang, H., Wang, N., and Yeung, D.
\newblock Collaborative deep learning for recommender systems.
\newblock In \emph{KDD}, pp.\  1235--1244, 2015.

\bibitem[Wang et~al.(2019{\natexlab{a}})Wang, Mao, He, Zhao, Jaakkola, and
  Katabi]{BIN}
Wang, H., Mao, C., He, H., Zhao, M., Jaakkola, T.~S., and Katabi, D.
\newblock Bidirectional inference networks: A class of deep bayesian networks
  for health profiling.
\newblock In \emph{AAAI}, volume~33, pp.\  766--773, 2019{\natexlab{a}}.

\bibitem[Wang et~al.(2019{\natexlab{b}})Wang, Zhou, and
  Bilmes]{BiasAttribution}
Wang, S., Zhou, T., and Bilmes, J.
\newblock Bias also matters: Bias attribution for deep neural network
  explanation.
\newblock In \emph{International Conference on Machine Learning}, pp.\
  6659--6667. PMLR, 2019{\natexlab{b}}.

\bibitem[Wang et~al.(2020)Wang, Menkovski, Wang, Du, and Pechenizkiy]{ICL}
Wang, Y., Menkovski, V., Wang, H., Du, X., and Pechenizkiy, M.
\newblock Causal discovery from incomplete data: A deep learning approach.
\newblock 2020.

\bibitem[Weinberger et~al.(2020)Weinberger, Janizek, and Lee]{LearningPrior}
Weinberger, E., Janizek, J., and Lee, S.-I.
\newblock Learning deep attribution priors based on prior knowledge.
\newblock \emph{Advances in Neural Information Processing Systems},
  33:\penalty0 14034--14045, 2020.

\bibitem[Yang et~al.(2022)Yang, Yuan, Zhang, Wang, Chen, Liu, Tarolli, Crepeau,
  Bukartyk, Junna, et~al.]{yang2022artificial}
Yang, Y., Yuan, Y., Zhang, G., Wang, H., Chen, Y.-C., Liu, Y., Tarolli, C.~G.,
  Crepeau, D., Bukartyk, J., Junna, M.~R., et~al.
\newblock Artificial intelligence-enabled detection and assessment of
  parkinson’s disease using nocturnal breathing signals.
\newblock \emph{Nature medicine}, 28\penalty0 (10):\penalty0 2207--2215, 2022.

\bibitem[Zhang et~al.(2018{\natexlab{a}})Zhang, Cui, Mueller, Tao, Kim,
  Rueschman, Mariani, Mobley, and Redline]{MESA}
Zhang, G.-Q., Cui, L., Mueller, R., Tao, S., Kim, M., Rueschman, M., Mariani,
  S., Mobley, D., and Redline, S.
\newblock The national sleep research resource: towards a sleep data commons.
\newblock \emph{JAMA}, 25\penalty0 (10):\penalty0 1351--1358,
  2018{\natexlab{a}}.

\bibitem[Zhang et~al.(2018{\natexlab{b}})Zhang, Cui, Mueller, Tao, Kim,
  Rueschman, Mariani, Mobley, and Redline]{SHHS1}
Zhang, G.-Q., Cui, L., Mueller, R., Tao, S., Kim, M., Rueschman, M., Mariani,
  S., Mobley, D., and Redline, S.
\newblock The national sleep research resource: Towards a sleep data commons.
\newblock \emph{Journal of the American Medical Informatics Association}, pp.\
  572--572, 08 2018{\natexlab{b}}.
\newblock \doi{10.1145/3233547.3233725}.

\bibitem[Zhao et~al.(2021)Zhao, Hoti, Wang, Raghu, and
  Katabi]{zhao2021assessment}
Zhao, M., Hoti, K., Wang, H., Raghu, A., and Katabi, D.
\newblock Assessment of medication self-administration using artificial
  intelligence.
\newblock \emph{Nature medicine}, 27\penalty0 (4):\penalty0 727--735, 2021.

\end{thebibliography}
\bibliographystyle{icml2023}

\newpage
\appendix
\onecolumn

\begin{center}
\LARGE{\textbf{Supplementary Material}}
\end{center}

\section{Evidence Lower Bound}\label{sec:app_elbo}
We can derive the evidence lower bound (ELBO) in~\eqnref{eq:elbo_simple} as follows: 
\begin{align*}
&\log p(\y|\x) \\
&=\log \int_{\z}\int_{\u} p(\y, \z,\u | \x) d \u d \z \\
&\geq \mathbb{E}_{q_\phi (\u, \z, \y| \x)} \left [ \log\frac{p_\theta(\y, \u, \z | \x)}{q_\phi(\u, \z, \y| \x)}\right ] \\
&=\mathbb{E}_{q_\phi(\u, \z, \y| \x)} \left [\log\frac{p(\u)  p_\theta(\z | \u, \x)p_\theta(\y| \u, \z)}{q_\phi(\u, \z, \y| \x)} \right ] \\
&=\EB_{q_\phi(\y, \u_l,\u_g,\z|\x)}[p_{\theta}(\y, \u_l,\u_g,\z|\x)] - \EB_{q_\phi(\y, \u_l,\u_g,\z|\x)}[q_\phi(\y, \u_l,\u_g,\z|\x)]\\
&=\mathbb{E}_{q_\phi(\u, \z, \y| \x)}[\log p(\u)]
+\mathbb{E}_{q_\phi(\u, \z, \y| \x)}[\log p_\theta(\z | \u, \x)] 
+\mathbb{E}_{q_\phi(\u, \z, \y| \x)}[\log p_\theta(\y | \u, \z)] 
-\mathbb{E}_{q_\phi(\u, \z, \y| \x)}[\log q_\phi(\u, \z, \y| \x)] \label{eq:marginal}
\end{align*}

\section{More Details on Counterfactual Inference}
\subsection{Proof on Identifiability} \label{sec:proof}
\begin{theorem}[\textbf{Identifiability}]
    Given the posterior distribution of exogenous variable $p(\u_l,\u_g|\x,\y)$, the effect of action $p(\y=\y^{cf} | do(\x=\x^{\prime}), \u_l,\u_g)$ can be identified using $\EB_{p(\z | \x^{\prime}, \u_l,\u_g)} p\left(\y^{cf} | \z, \u_l,\u_g\right)$.
\end{theorem}

\begin{proof}
With $\u=(\u_l,\u_g)$ and applying Rule 2 and 3 in do-calculus~\cite{Causality}, we have
\begin{align}
&p\left(\y^{cf} | do\left(\x=\x^{\prime}\right), \u\right) \nonumber\\
&\ \  =\int_{\z} p\left(\y^{cf} | \z, do\left(\x^{\prime}\right), \u\right) p\left(\z | do\left(\x^{\prime}\right), \u\right) d\z \nonumber\\
&\ \  =\EB_{p\left(\z | do\left(\x^{\prime}\right), \u\right)} \left [ p\left(\y^{cf} | \z, do\left(\x^{\prime}\right), \u\right) \right ]\nonumber\\
&\overset{\mathrm{Rule 2}}{=}\EB_{p\left(\z | \x^{\prime}, \u\right)} \left [p\left(\y^{cf} | do(\z), do\left(\x^{\prime}\right), \u\right) \right ]\nonumber\\
&\overset{\mathrm{Rule 3}}{=}\EB_{p\left(\z | \x^{\prime}, \u\right)} \left [p\left(\y^{cf} | do(\z), \u\right) \right ]\nonumber\\
&\overset{\mathrm{Rule 2}}{=}\EB_{p\left(\z | \x^{\prime}, \u\right)} \left [p\left(\y^{cf} | \z, \u\right) \right ],
\end{align}
concluding the proof. 
\end{proof}

\subsection{Counterfactual Explanation Algorithm} \label{sec:app_algo}
The pseudo-code for counterfactual explanation is shown in \algref{alg:cf}.
\begin{algorithm}[h]
   \caption{Generating Counterfactual Explanations}
   \label{alg:cf}
\begin{algorithmic}
   \STATE {\bfseries Input:} data $\x$, threshold $\epsilon$, number of samples $n$ and $m$.
   \STATE Sample $\y^{pred}$ from $q_\phi(\y|\x)$.
   \STATE Set a target counterfactual output $\y^{cf} \neq \y^{pred}$
   \FOR{$i=1$ {\bfseries to} $m$}
   \STATE Sample $\u_i$ from $q_\phi(\u|\x,\y)$
    \FOR{$j=1$ {\bfseries to} $n$}
    \STATE Sample $\z_{ij} \sim p_\theta(\z|\x, \u_i)$
    \STATE Sample $\y_{ij} \sim p_\theta(\y|\z_{ij}, \u_i)$
    \ENDFOR
   \ENDFOR
   \STATE  Calculate $\y^{cf} = \frac{\sum_i \sum_j \y_{ij}}{n\times m}$
   \STATE Set $\x^{cf} = \x$
   \WHILE{$|\y-\y^{cf}| \geq \epsilon$}
    \STATE Update $\x^{cf} = \x^{cf} - \lambda \frac{\partial \mid \y-\y^{cf} \mid}{\partial \x^{cf}}$
   \ENDWHILE 
   \STATE{\bfseries return} $\x^{cf}$
\end{algorithmic}
\end{algorithm}

\section{More Details on Experiments}

\subsection{Details on Datasets}\label{sec:app_dataset}
\textbf{\emph{Spike} Dataset.} The generation process of the \emph{Spike} dataset is summarized below:

We generate $D=3$ independent channels of non–linear auto-regressive moving average (NARMA) time series data using the following formula:
\begin{equation}
x_{d,t+1}=0.5x_{d,t}+0.5x_{d,t}\sum_{i=0}^{l-1}x(t-l)+1.5u(t-(l-1))u(t)+0.5+\alpha_{d}t \label{eq:spike_narma}
\end{equation}
for $t=[1,\dots,80]$, order $l=2$, $u \sim \NM(0, 0.03)$, and $\alpha_d$ is set differently for each channel ($\alpha_1 = 0.1$, $\alpha_2 = 0.065$ and $\alpha_3 = 0.003$). We use $d$ to index the $3$ channels. 

\begin{equation}
\begin{aligned}
&\mathbf{\tha} = [0.8, 0.4, 0] ; \\
&n_{d} \sim \operatorname{Bernoulli}(\mathbf{\theta}_{d}) ; \\
&m_{d} \sim \operatorname{Bernoulli}(\mathbf{\theta}_{d}) ; \\
&\eta_{d}= \begin{cases}\operatorname{Poisson}(\lambda=2) &  \mathbf{1} \text { if } \left(n_{d}==1\right) \\
0 & \text { otherwise }\end{cases} \\
&\mathbf{g}_{d} = \operatorname{Sample}\left([T], \eta_{d}\right)\\ 
&x_{d, t}=x_{d, t}+(\kappa_{d,t}+ \theta_{d})  \quad \text{where } \kappa_{d,t}\sim \NM(1, 0.3)  \quad \forall t \in \mathbf{g}_{d} \label{synthetic:2} \\
& y_{t}=\left\{\begin{array}{ll}
0 & t \leqslant \min \left(\mathbf{g}_d\right) \text {where  } m_d=1 \\
1 & \text{otherwise}
\end{array}\right.
\end{aligned}
\end{equation}

\textbf{Real-World Datasets.} The \emph{real-world} datasets also include additional patient information such as age, gender and race. The age range is $[44, 90]$ for SHHS, $[54, 95]$ for MESA and $[71, 90]$ for SOF. In total, there are $2{,}651$, $2{,}055$ and $453$ patients in SHHS, MESA, and SOF, respectively, with an average of $1{,}043$ breathing signal segments (approximately $8.7$ hours). 

\subsection{Details on Baselines}\label{sec:app_baseline}

We use the following five types of state-of-the-art baselines: 
\begin{itemize}
    \item Gradient-based methods. Regularized Gradient Descent (\textbf{RGD})~\cite{RGD} directly models $p(y|x)$ and provide the explanation by modifying input with gradients along with L1 regularization; it is therefore it is an actionable explanation method. Gradient-weighted Class Activation Mapping (\textbf{GradCAM})~\cite{GradCAM} is originally designed to models with convolutional layers. We added convolutional layers to our model to adapt GradCAM for our time series data. Gradient SHapley Additive exPlanations (\textbf{GradSHAP})~\cite{GradSHAP} is a game theoretic approach that uses the expectation of gradients to approximate the SHAP values.

    \item Perturbation-based method. Local Interpretable Model-agnostic Explanations (\textbf{LIME})~\cite{LIME} is a local interpretation approach based on the local linearity assumption and provides explanations by fitting the output of the model given locally perturbed input.

    \item Distribution shift method. Feature Importance in Time (\textbf{FIT})~\cite{FIT} is a time-series back-box model explanation method that evaluates the importance of the input data based on the temporal distribution shift and unexplained distribution shift.

    \item Association rule method. Case-crossover APriori (\textbf{CAP})~\cite{CAP} applies association rule mining algorithm, Apriori, to explore the causal relationship in time-series data.

    \item Generative method. Counterfactual Residual Generative Adversarial Network (\textbf{CounteRGAN})~\cite{CounteRGAN} combines a Residual GAN (RGAN) and a classifier to generate counterfactual output. The generated residual output is considered as the result of the $do$ operation. CounteRGAN is an actionable explanation method.
\end{itemize}

\section{Additional Results}\label{sec:app_add}
\textbf{Results on the \emph{Spike} Dataset.} 
Note that methods such as FIT are \emph{not} an actionable explanation method. It cannot provide counterfactual explanation $\x^{cf}$ that could shift the model prediction from $\y^{pred}$ to $\y^{cf}$; it can only provide importance scores to explain the current prediction $\y^{pred}$. \figref{fig:fitcase} shows the importance scores produced by FIT to explain the same model in~\figref{fig:syntheticcase} in the \emph{Spike} dataset. 

\begin{figure}[!t]
\vspace{10pt}
    \centering
        \includegraphics[width=1\textwidth]{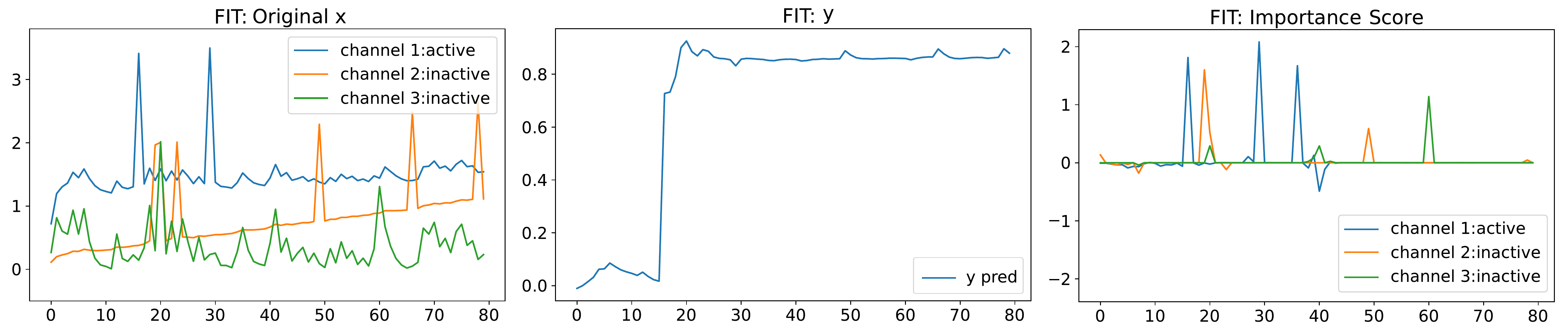}
        \vspace{-25pt}
    \caption{\label{fig:fitcase}Qualitative results from FIT in the \emph{Spike} dataset.}
    \vspace{-10pt}
\end{figure}

\begin{figure}[!t]
\vspace{10pt}
    \centering
        \includegraphics[width=1\textwidth]{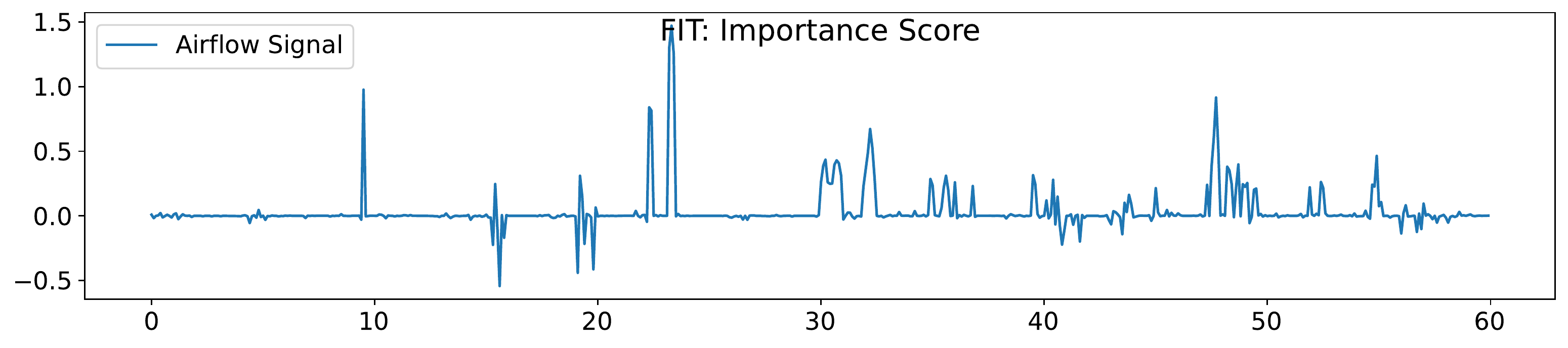}
        \vspace{-25pt}
    \caption{\label{fig:fitcase_real}Qualitative results on FIT in a \emph{MESA} intra-dataset setting.}
    \vspace{10pt}
\end{figure}

\textbf{Results on Real-World Datasets.} As FIT cannot provide actionable explanations, the importance scores produced by FIT to explain the same model in~\figref{fig:realcase} in \emph{real-world dataset} are separately shown in \figref{fig:fitcase_real}.

\tabref{tab:sof}$\sim$\ref{tab:tosof} show more detailed CCR results with different counterfactual action settings and dataset settings. \tabref{tab:realcfaccu} shows detailed counterfactual accuracy for each dataset settings and each counterfactual action settings. 

\begin{table}[h]
\centering
\caption{CCR for different counterfactual action settings (e.g., A$\rightarrow$D means `Awake'$\rightarrow$`Deep Sleep') in \emph{SOF}.}
\vspace{3pt}
\label{tab:sof}
\resizebox{0.7\textwidth}{!}{%
\begin{tabular}{cccccccc}
\toprule
 & A$\rightarrow$D & L$\rightarrow$D & R$\rightarrow$D & R$\rightarrow$A & D$\rightarrow$A & L$\rightarrow$A & Average \\ \midrule
CounteRGAN & \textbf{1.356} & {\ul 1.138} & 1.165 & 1.127 & 0.908 & 1.060 & {\ul 1.126} \\
RGD & 0.996 & 1.103 & {\ul 1.174} & 0.951 & 1.108 & {\ul 1.194} & 1.088 \\
GradCAM & 1.012 & 0.926 & 0.644 & 0.727 & 1.012 & 1.003 & 0.887 \\
GradSHAP & 0.843 & 1.177 & {\ul 1.219} & 0.913 & 0.843 & 0.981 & 0.996 \\
LIME & {\ul 1.220} & 0.857 & 1.095 & 0.936 & {\ul 1.220} & 1.139 & 1.078 \\
FIT & 1.064 & 1.150 & 1.169 & {\ul 1.202} & 1.064 & 0.667 & 1.053 \\
CAP & 0.989 & 0.952 & 1.019 & 1.139 & 0.989 & 0.689 & 0.963 \\ 
CounTS (Ours) & 1.179 & \textbf{1.161} & \textbf{1.240} & \textbf{1.475} & \textbf{1.245} & \textbf{1.387} & \textbf{1.281} \\ \bottomrule
\end{tabular}%
}
\end{table}

\begin{table}[h]
\centering
\caption{CCR for different counterfactual action settings in \emph{SHHS+SOF$\rightarrow$ MESA}.}
\label{tab:tomesa}
\resizebox{0.7\textwidth}{!}{%
\begin{tabular}{cccccccc}
\toprule
 & A$\rightarrow$D & L$\rightarrow$D & R$\rightarrow$D & R$\rightarrow$A & D$\rightarrow$A & L$\rightarrow$A & Average \\ \midrule
CounteRGAN & {\ul 1.381} & \textbf{1.379} & 1.068 & 1.373 & {\ul 1.355} & 1.158 & {\ul 1.285} \\
RGD & 1.226 & 1.130 & 1.089 & 1.215 & 1.177 & {\ul 1.293} & 1.188 \\
GradCAM & 1.056 & 1.014 & 0.743 & 1.218 & 1.056 & 1.101 & 1.031 \\
GradSHAP & 1.174 & 0.822 & 1.046 & 0.948 & 1.174 & 0.988 & 1.025 \\
LIME & 0.904 & 1.246 & 1.138 & 1.017 & 0.904 & 0.836 & 1.007 \\
FIT & 1.143 & 1.228 & {\ul 1.193} & 1.206 & 1.143 & 0.938 & 1.142 \\
CAP & 0.969 & 1.179 & 0.947 & {\ul 1.418} & 0.969 & 0.703 & 1.031 \\
CounTS (Ours) & \textbf{1.391} & {\ul 1.331} & \textbf{1.253} & \textbf{1.470} & \textbf{1.465} & \textbf{1.355} & \textbf{1.377}\\ \bottomrule
\end{tabular}%
}
\end{table}

\begin{table}[h]
\vspace{0pt}
\centering
\caption{CCR for different counterfactual action settings in \emph{MESA}.}
\label{tab:mesa}
\resizebox{0.7\textwidth}{!}{%
\begin{tabular}{cccccccc}
\toprule
 & A$\rightarrow$D & L$\rightarrow$D & R$\rightarrow$D & R$\rightarrow$A & D$\rightarrow$A & L$\rightarrow$A & Average \\ \midrule
 CounteRGAN & {\ul 1.183} & \textbf{1.345} & 1.160 & 1.117 & 0.974 & 1.200 & 1.163 \\
RGD & 1.159 & 1.088 & 1.149 & 1.204 & 1.190 & {\ul 1.207} & {\ul 1.166} \\
GradCAM & 1.078 & 0.790 & 0.419 & 0.954 & 1.078 & 0.977 & 0.883 \\
GradSHAP & 0.885 & 1.027 & 1.039 & {\ul 1.273} & 0.885 & 1.145 & 1.042 \\
LIME & 0.903 & 1.003 & 0.930 & 1.233 & 0.903 & 0.952 & 0.987 \\
FIT & 1.210 & 1.118 & 1.024 & 1.082 & {\ul 1.210} & 0.887 & 1.089 \\
CAP & 0.786 & 1.001 & {\ul 1.176} & 1.243 & 0.786 & 0.783 & 0.962 \\ 
CounTS (Ours) & \textbf{1.213} & {\ul 1.309} & \textbf{1.294} & \textbf{1.526} & \textbf{1.350} & \textbf{1.285} & \textbf{1.329}\\ \bottomrule
\end{tabular}%
}
\vspace{0pt}
\end{table}

\begin{table}[h]
\vspace{10pt}
\centering
\caption{CCR for different counterfactual action settings in \emph{MESA+SOF$\rightarrow$SHHS}.}
\label{tab:toshhs}
\resizebox{0.7\textwidth}{!}{%
\begin{tabular}{cccccccc}
\toprule
 & A$\rightarrow$D & L$\rightarrow$D & R$\rightarrow$D & R$\rightarrow$A & D$\rightarrow$A & L$\rightarrow$A & Average \\ \midrule
 CounteRGAN & \textbf{1.317} & 1.157 & \textbf{1.281} & 1.300 & 1.270 & 1.060 & 1.231 \\
RGD & 1.193 & {\ul 1.266} & 1.115 & {\ul 1.313} & {\ul 1.304} & {\ul 1.255} & {\ul 1.241} \\
GradCAM & 1.131 & 1.093 & 0.725 & 1.306 & 1.131 & 0.953 & 1.056 \\
GradSHAP & 1.025 & 0.910 & 1.092 & 0.886 & 1.025 & 0.910 & 0.975 \\
LIME & 1.018 & 1.092 & 1.118 & 0.944 & 1.018 & 0.930 & 1.020 \\
FIT & 1.270 & 1.259 & 1.105 & 1.263 & 1.270 & 0.991 & 1.193 \\
CAP & 0.994 & 0.819 & 1.033 & 0.953 & 0.994 & 0.777 & 0.929 \\ 
CounTS (Ours) & {\ul 1.301} & \textbf{1.452} & {\ul 1.167} & \textbf{1.560} & \textbf{1.417} & \textbf{1.443} & \textbf{1.390}\\ \bottomrule
\end{tabular}%
}
\end{table}

\begin{table}[h]
\centering
\vspace{0pt}
\caption{CCR for different counterfactual action settings in \emph{SHHS+MESA$\rightarrow$SOF}.}
\label{tab:tosof}
\resizebox{0.7\textwidth}{!}{%
\begin{tabular}{cccccccc}
\toprule
 & A$\rightarrow$D & L$\rightarrow$D & R$\rightarrow$D & R$\rightarrow$A & D$\rightarrow$A & L$\rightarrow$A & Average \\ \midrule
CounteRGAN & \textbf{1.210} & 1.050 & 1.034 & \textbf{1.403} & 1.065 & 1.248 & {\ul 1.153} \\
RGD & 1.183 & 0.993 & 1.015 & 1.260 & {\ul 1.155} & 1.201 & 1.135 \\
GradCAM & 1.108 & 1.018 & 1.020 & 1.215 & 1.108 & {\ul 1.286} & 1.126 \\
GradSHAP & 0.983 & 0.929 & 0.780 & 1.044 & 0.983 & 0.728 & 0.908 \\
LIME & 0.738 & 1.051 & {\ul 1.199} & 1.155 & 0.738 & 0.708 & 0.932 \\
FIT & 1.110 & {\ul 1.175} & 0.954 & 1.206 & 1.110 & 0.876 & 1.072 \\
CAP & 1.106 & 0.951 & 0.833 & 1.145 & 1.106 & 1.060 & 1.034 \\ 
CounTS (Ours) & {\ul 1.203} & \textbf{1.266} & \textbf{1.202} & {\ul 1.395} & \textbf{1.472} & \textbf{1.341} & \textbf{1.313} \\ \bottomrule
\end{tabular}%
}
\vspace{10pt}
\end{table}

\begin{table}[t]
\vspace{-9cm}
\centering
\caption{Couterfactual accuracy for each dataset settings and each counterfactual action settings.}
\label{tab:realcfaccu}
\resizebox{0.7\textwidth}{!}{%
\begin{tabular}{ccccccccc}
\toprule
 & Method  & A$\rightarrow$D & L$\rightarrow$D & R$\rightarrow$D & R$\rightarrow$A & D$\rightarrow$A & L$\rightarrow$A & Average\\ \midrule\midrule
\multirow{4}{*}{\tabincell{c}{SHHS+MESA\\$\downarrow$\\SOF}}
& RGD  & {\ul 0.867} & {\ul 0.903} & {\ul 0.921} & \textbf{0.932} & \textbf{0.931} & \textbf{0.922} & {\ul 0.913}\\
& CounteRGAN  & 0.848 & \textbf{0.917} & 0.903 & 0.917 & 0.896 & 0.900 & 0.897\\
 & FIT  & 0.285 & 0.296 & 0.328 & 0.384 & 0.285 & 0.348 & 0.331\\ 
  & CounTS (Ours)  & \textbf{0.868} & 0.893 & \textbf{0.929} & {\ul 0.918} & {\ul 0.920} & {\ul 0.913} & \textbf{0.916} \\\midrule
  \multirow{4}{*}{\tabincell{c}{MESA+SOF\\$\downarrow$\\SHHS}}
& RGD  & {\ul 0.828} & \textbf{0.882} &  0.916 & \textbf{0.910} & \textbf{0.903} & \textbf{0.890} & \textbf{0.887}\\
 & CounteRGAN  & {\ul 0.828} & 0.869 & {\ul 0.920} & 0.867 & 0.880 & {\ul 0.875} & {\ul 0.868}\\
 & FIT  & 0.334 & 0.263 & 0.374 & 0.354 & 0.334 & 0.348 & 0.327\\ 
 & CounTS (Ours)  & \textbf{0.837} & {\ul 0.878} & \textbf{0.923} & {\ul 0.899} & {\ul 0.899} & \textbf{0.890} & \textbf{0.887}\\  \midrule
 \multirow{4}{*}{\tabincell{c}{SHHS+SOF\\$\downarrow$\\MESA}} 
& RGD  & \textbf{0.882} & \textbf{0.898} & {\ul 0.933} & \textbf{0.910} & \textbf{0.938} & {\ul 0.888} & \textbf{0.907}\\
 & CounteRGAN  & {\ul 0.872} & 0.878 & 0.926 & 0.905 & 0.901 & \textbf{0.897} & 0.880\\
 & FIT  & 0.335 & 0.269 & 0.301 & 0.357 & 0.335 & 0.397 & 0.346\\ 
 & CounTS (Ours)  & \textbf{0.882} & {\ul 0.888} & \textbf{0.936} & \textbf{0.910} & {\ul 0.930} & 0.881 & {\ul 0.889}\\
\midrule\midrule
\multirow{4}{*}{SOF} 
& RGD &   {\ul 0.823} & {\ul 0.820} & \textbf{0.919} & {\ul 0.856} & {\ul 0.861} & \textbf{0.891} &\textbf{0.862}\\
 & CounteRGAN  & \textbf{0.831} & 0.800 & 0.899 & \textbf{0.867} & 0.832 & 0.875 & 0.844\\
 & FIT  & 0.257 & 0.328 & 0.311 & 0.328 & 0.257 & 0.310 & 0.248\\ 
 & CounTS (Ours)  & 0.813 & \textbf{0.831} & \textbf{0.907} & 0.843 & \textbf{0.864} & {\ul 0.881} & {\ul 0.852} \\ \midrule
\multirow{4}{*}{SHHS} 
& RGD  & \textbf{0.880} & \textbf{0.892} & \textbf{0.940} & \textbf{0.944} & {\ul 0.887} & \textbf{0.900} & \textbf{0.907}\\
 & CounteRGAN  & 0.841 & 0.854 & 0.916 & 0.913 & \textbf{0.888} & 0.875 & 0.888\\
 & FIT  & 0.335 & 0.272 & 0.353 & 0.337 & 0.335 & 0.348 & 0.322\\ 
 & CounTS (Ours)  & {\ul 0.868} & {\ul 0.880} & {\ul 0.930} & \textbf{0.931} & 0.875 & {\ul 0.889} & {\ul 0.903} \\\midrule
\multirow{4}{*}{MESA} 
& RGD  & \textbf{0.846} & \textbf{0.888} & {\ul 0.912} & \textbf{0.907} & \textbf{0.909} & {\ul 0.888} & \textbf{0.892}\\
 & CounteRGAN  & 0.828 & 0.869 & \textbf{0.920} & 0.867 & 0.880 & 0.875 & 0.868\\
 & FIT  & 0.290 & 0.316 & 0.320 & 0.307 & 0.290 & 0.301 & 0.294\\
 & CounTS (Ours)  & {\ul 0.837} & {\ul 0.878} & 0.903 & {\ul 0.899} & {\ul 0.899} & \textbf{0.890} & {\ul 0.887} \\
\bottomrule
\end{tabular}%
}
\end{table}


\end{document}